\documentclass{amsart}

\usepackage{mathtools}
\usepackage{mathabx}
\usepackage{booktabs}
\usepackage[T1]{fontenc}
\usepackage{lmodern}
\usepackage{amsmath,amsfonts,mathrsfs,amssymb,amsthm}
\usepackage{enumerate}
\usepackage{bm}
\usepackage{dsfont}
\usepackage{stmaryrd}
\usepackage{wrapfig,xfrac}
\usepackage{tikz}
\usepackage{tikz-3dplot}
\definecolor{myblue}{rgb}{0.21, 0.34, 0.74}
\definecolor{mygrey}{rgb}{0.55, 0.57, 0.67}
\definecolor{myred}{rgb}{0.79, 0.0, 0.09}
\definecolor{mygreen}{rgb}{0.05, 0.5, 0.06}

\usepackage[colorlinks=true, citecolor=myblue, linkcolor=black]{hyperref}

\DeclareMathAlphabet{\mathscrbf}{OMS}{mdugm}{b}{n}
\usepackage[scr=boondoxo]{mathalpha}
\usepackage{cleveref}
\usepackage{subcaption}
\DeclareCaptionLabelFormat{empty}{}
\usepackage{comment}

\usepackage{wrapfig}
\usepackage{upgreek}
\usepackage{natbib}

\newcommand{\proj}{\mathbf{P}}

\newcommand{\cK}{\mathcal{K}}

\newcommand{\R}{{\rm I}\kern-0.18em{\rm R}}
\newcommand{\h}{{\rm I}\kern-0.18em{\rm H}}
\newcommand{\K}{{\rm I}\kern-0.18em{\rm K}}
\newcommand{\p}{{\rm I}\kern-0.18em{\rm P}}
\newcommand{\E}{{\rm I}\kern-0.18em{\rm E}}

\newcommand{\1}{{\rm 1}\kern-0.24em{\rm I}}
\newcommand{\N}{{\rm I}\kern-0.18em{\rm N}}

\newcommand{\att}{\mathsf{\mathop{ATT}}}

\newcommand{\ala}{\alpha}

\newcommand{\bea}{\beta}

\newcommand{\cln}{\mathcal{N}}
\newcommand{\clatt}{\mathcal{ATT}}

\newcommand{\tr}{\mathrm{Tr}}
\newcommand{\sumv}{\mathbf{V}}
\newcommand{\sumt}{\mathbf{W}}
\newcommand{\bfr}{\mathbf{R}}
\newcommand{\bfu}{\mathbf{U}}
\newcommand{\bfo}{\mathbf{o}}

\numberwithin{equation}{section}

\newcommand{\nbzl}[1]{ {\colorbox{myred}{\color{white} \textsf{ZL}} \color{myred}{#1}} }

\usepackage[font=small,labelfont=bf]{caption}

\theoremstyle{plain}
\newtheorem{theorem}{Theorem}[section]
\newtheorem{assumption}{Assumption}

\newtheorem{proposition}[theorem]{Proposition}
\newtheorem{lemma}[theorem]{Lemma}

\newtheorem{remark}[theorem]{Remark}

\begin{document}

\title{Critical attention scaling in long-context transformers}





 \author[S. Chen]{Shi Chen}
 \address{(SC) Department of Mathematics, Massachusetts Institute of Technology, 77 Massachusetts Ave, 02139 Cambridge MA, USA} 
\email{schen636@mit.edu}

 \author[Z. Lin]{Zhengjiang Lin}
\address{(ZL) Department of Mathematics, Massachusetts Institute of Technology, 77 Massachusetts Ave, 02139 Cambridge MA, USA} 
\email{linzj@mit.edu}

 \author[Y. Polyanskiy]{Yury Polyanskiy}
 \address{(YP) Department of Electrical Engineering and Computer Science, Massachusetts Institute of Technology, 77 Massachusetts Ave, 02139 Cambridge MA, USA}
 \email{yp@mit.edu}

 \author[P. Rigollet]{Philippe Rigollet}
 \address{(PR) Department of Mathematics, Massachusetts Institute of Technology, 77 Massachusetts Ave, 02139 Cambridge MA, USA} 
 \email{rigollet@math.mit.edu}

\date{}

\begin{abstract}
As large language models scale to longer contexts, attention layers suffer from a fundamental pathology: attention scores collapse toward uniformity as context length $n$ increases, causing tokens to cluster excessively, a phenomenon known as rank-collapse. While \emph{attention scaling} effectively addresses this deficiency by rescaling attention scores with a polylogarithmic factor $\beta_n$, theoretical justification for this approach remains lacking.

We analyze a simplified yet tractable model that magnifies the effect of attention scaling. In this model, attention exhibits a phase transition governed by the scaling factor $\beta_n$: insufficient scaling collapses all tokens to a single direction, while excessive scaling reduces attention to identity, thereby eliminating meaningful interactions between tokens.
Our main result identifies the critical scaling $\beta_n \asymp \log n$ and provides a rigorous justification for attention scaling in YaRN and Qwen, clarifying why logarithmic scaling maintains sparse, content-adaptive attention at large context lengths.
\end{abstract}

\maketitle

\setcounter{tocdepth}{2}
\makeatletter
\def\l@subsection{\@tocline{2}{0pt}{2.8pc}{5pc}{}}

\tableofcontents

\newpage
\section{Introduction}\label{sec:intro}

The attention mechanism is a cornerstone of modern transformer architectures on which Large Language Models (LLMs) rely. Mathematically, an attention layer is a nonlinear operator $\att$ that maps a collection of \emph{tokens} $\{x_1, \ldots, x_n\}$ from $\R^d$ to $\R^d$. This operator is parametrized by three (possibly sparse) $d$ by $d$ matrices $K,Q,$ and $V$ and maps $\{x_1, \ldots, x_n\}$ to $\{x'_1, \ldots, x'_n\}$ using the following formula. Define the normalization operator $N(x)=x/\|x\|$ and for any $i=1, \ldots, n$ define $q_i=QN(x_i)$, $k_i=KN(x_i)$. Then $x_i'=\att(x_1, \ldots, x_n)_i$ is defined as
\begin{equation}
    \label{eq:att}
    x_i'= V\sum_{j=1}^n N(x_j)A_{ij}\,, \qquad A_{ij}=\frac{e^{a_{ij}}}{\sum_{k=1}^n e^{a_{ik}}}\,,
\end{equation}
where the terms $a_{ij}=q_i^\top k_j$ are called \emph{attention scores}.

A recent line of theoretical work has demonstrated that attention acts as a \emph{contractive} operator that tends to cluster tokens together; see~\cite{dong2021attention, geshkovski2023emergence, geshkovski2023mathematical, karagodin2024clustering,geshkovski2024dynamic,bruno2024emergence,andrew25,chen2025quantitative, cowsik2024geometric,giorlandino2025failuremodesdeeptransformers,rigollet2025mean}. This clustering effect is also known as ``rank-collapse" or ``token uniformity" and arises because the distribution of attention scores tends to flatten as the sequence length $n$ grows, causing each token to disperse its attention across too many other tokens rather than focusing selectively.

Various practical solutions have been proposed to curb this clustering behavior. In this work, we focus on simple context-length-aware modifications of the attention mechanism following ideas practically implemented as YaRN~\citep{peng2023yarn}, Qwen~\citep{bai2023qwen}, SSMax~\citep{nakanishi2025scalable}, and SWAN-GPT~\citep{puvvada2025swan}. These methods employ a straightforward strategy that rescales attention scores $a_{ij}$ by a single poly-logarithmic factor $\beta_n$; see Table~\ref{tab:scalings}. Our goal in this paper is to answer the following fundamental question:

\begin{center}
  {\it What is the optimal order of magnitude of the $\beta_n$ scaling?}
\end{center}

\begin{wraptable}{r}{0.4\textwidth}
\vspace{-10pt}
\centering
\begin{tabular}{lc}
\toprule
Method & $\beta_n$ scaling \\
\midrule
YaRN & $(\log n)^2$ \\
Qwen & $\log n$ \\
SSMax & $\log n$ \\
SWAN-GPT & $\log n$ \\
\bottomrule
\end{tabular}
\caption{Attention scaling factors for various methods. The standard attention score $\exp(k_i^\top q_j)$ is replaced with $\exp(C\beta_n k_i^\top q_j)$, $C>0$.
}
\label{tab:scalings}
\vspace{-10pt}
\end{wraptable}

To address this question, we propose a highly simplified yet completely tractable model for attention. This model exhibits a phase transition governed by the parameter $\beta_n$ as $n\to \infty$: when $\beta_n$ is below a critical threshold, attention becomes overly contractive and collapses all tokens to a single direction, while when $\beta_n$ is too large, attention acts as an identity operator and fails to process information effectively. More precisely, we establish that the critical parameter $\beta_n$ scales as $\log n$, which corroborates the empirical guidelines underlying YaRN, Qwen, SSMax, and SWAN-GPT.

Our work is intimately connected to the recent contributions of \cite{giorlandino2025failuremodesdeeptransformers} and \cite{cowsik2024geometric}, who investigate the contractive effects of attention mechanisms with random key and query matrices $K$ and $Q$ to establish proper initialization schemes for these parameters. A crucial insight from~\cite{cowsik2024geometric} is that analyzing the evolution of symmetric token configurations provides a more mathematically tractable framework compared to the generic input distributions considered in~\cite{geshkovski2023mathematical}. This symmetric setting, while simplified, captures essential dynamics of the attention mechanism and enables rigorous theoretical analysis; see also~\cite{karagodin2025normalization}.

The choice $\beta_n = \gamma \log n$ appears natural in retrospect. As noted in~\cite{nakanishi2025scalable}, with such a scaling the attention weights $A_{ij}$ in Equation~(\ref{eq:att}) become
$$
A_{ij}= \frac{n^{\gamma a_{ij}}}{\sum_{k=1}^n n^{\gamma a_{ik}}}\,.
$$
To illustrate the resulting dynamics, consider a simplified regime where all attention scores $a_{ij}$ are of order one: specifically, let $a_{ii}=1$ and $a_{ij}=\rho>0$ for $i\neq j$. In this setting, the off-diagonal weights satisfy
$$
A_{ij}= \frac{n^{\gamma\rho}}{n^\gamma + (n-1)n^{\gamma \rho}}
\sim
\left\{
\begin{array}{ll}
   \sfrac{1}{n} &\text{if } \gamma < \frac{1}{1-\rho}  \\
   \sfrac{1}{n^{\gamma(1-\rho)}} & \text{if } \gamma > \frac{1}{1-\rho}  \\
\end{array}\right.
$$
This analysis reveals two distinct regimes. When $\gamma$ is small (subcritical regime), attention weights are asymptotically uniform, resulting in diffuse attention that, as we demonstrate below, leads to severe token contraction. Conversely, when $\gamma$ is large (supercritical regime), off-diagonal weights become negligible with respect to the diagonal ones so that the attention mechanism is effectively suppressed.

The critical regime emerges at the phase boundary  $\gamma = \frac{1}{1-\rho}$  where     attention can concentrate on a sublinear yet nontrivial number of tokens so as to maintain sufficient connections to facilitate information flow from a small set of important tokens. This sparse attention is related to structured attention mechanisms employed in long-context architectures such as Longformer~\citep{beltagy2020longformer} and SWIN~\citep{liu2021swin} which implement  a sliding window over $k\ll n$-nearest neighbors but where proximity is measured in terms of token position rather than embedding. Unlike these structurally constrained approaches that rely on fixed positional neighborhoods, the logarithmic scaling enables the attention pattern to be entirely \emph{content-adaptive}, allowing each token to dynamically select its most relevant context based on semantic similarity rather than positional proximity.

Following similar motivations, \cite{giorlandino2025failuremodesdeeptransformers} 
establish a compelling analogy between attention dynamics and the random energy model from 
statistical physics~\citep{Derrida1981}. Using the replica method---an analytical heuristic 
from statistical physics---they identify a phase transition occurring at $\beta_n \sim \sqrt{\log n}$, 
which differs from the scalings presented in Table~\ref{tab:scalings}. This result represents 
a significant discrepancy from our findings and highlights fundamental differences in modeling assumptions.
More specifically, their approach assumes that the attention scores $a_{ij}$ are correlated 
Gaussian random variables. This assumption effectively induces a random geometry on the token 
space, where similarity between tokens is treated as fundamentally random. In this sense, 
their model bears closer resemblance to recent Kuramoto models on random graphs studied 
in~\cite{abdalla2022expander,JaiMizSaw25}, where the authors investigate the synchronization 
of oscillators interacting across the edges of a (sparse) Erd\H{o}s--R\'enyi random graph 
with unit edge weights.
However, in the case of~\cite{giorlandino2025failuremodesdeeptransformers}, the random graph 
is both directed and dense, with the edge pointing from token $j$ to token $i$ having weight 
given by 
\begin{equation}
A_{ij}=\frac{e^{\beta_n a_{ij}}}{\sum_{k=1}^n e^{\beta_na_{ik}}}
\end{equation}
where $a_{ij}$ are Gaussian random variables. While~\cite{giorlandino2025failuremodesdeeptransformers} assumes a specific correlation structure between the Gaussian random variables, the phase transition they uncover is expected to be universal within a large class of random matrices including Wigner ones. 
Crucially though, in such models, the interaction strength $A_{ij}$ 
is independent of the positional relationship between tokens $i$ and $j$, making this model 
qualitatively different from standard attention mechanisms where attention is focused on few (or all) of the preceding tokens. For completeness, we refer readers to Appendix~\ref{app:gaussian-supercritical} for a derivation of the critical scaling $\beta_n \asymp \sqrt{\log n}$ in the i.i.d.\ Gaussian score model.

\cite{bruno2025multiscaleanalysismeanfieldtransformers} adopt a different approach to studying the regime where $n \to \infty$ and $\beta_n \to \infty$, in a more general setting than ours.  By considering various levels of generality for the matrices $K, Q, V$, this work identifies distinct regimes of token dynamics and relates them to the hardmax ($\beta = \infty$) limit. Importantly, the analysis is conducted in the subcritical regime and differs from the present work in focusing on a broader class of models, for which the critical regime has yet to be precisely characterized. We believe that combining the analytical tools developed in both papers could yield a deeper understanding of this critical regime and represents a promising direction for future research.


The remainder of the paper is organized as follows. 
Section~\ref{sec:phase transition dynamic} provides a precise mathematical formulation of the phase transition phenomena for the rescaled attention layer. We begin by analyzing token angles and the contractive behavior of tokens under two settings: an idealized but intuitive simplex model (Section~\ref{sec:intro simplex}) and a more realistic model with the simplex constraint relaxed (Section~\ref{sec:equal angle}). In both cases, we identify three distinct regimes of the scaling parameter, each leading to qualitatively different contrastive behaviors of the self-attention layer. 
Section~\ref{sec:propagation results} turns to the gradient norm of the rescaled attention operator. Because rank collapse is often accompanied by vanishing gradients, we characterize the gradient dynamics across scaling regimes and show when gradients vanish, or stabilize to non-trivial limits. 
Section~\ref{sec:num} presents our numerical experiments, which validate these theoretical predictions. 

Throughout this paper, when we denote a quantity as $o_n(1)$, where $n$ is the number of tokens, we mean there are positive constants $C_1,C_2$ independent of the dimension $d$, such that $|o_n(1)| \leq C_1 n^{-C_2}$. The constants $C_1,C_2$ depend on the assumptions in theorems.

\section{A phase transition for attention}\label{sec:phase transition dynamic}

In this section, we establish the main theorem of this paper, namely a phase transition for the contractive properties of the attention layer when $\beta_n = \gamma \log n$ for some $\gamma>0$.

Following~\cite{geshkovski2023mathematical}, we study a simplified version of the attention layer with pre-layer norm that is described in the introduction by assuming that $K=Q=V=I_d$. More specifically, the model we study is given as follows.

For any two points $x,y \in \R^d$, let $\langle x, y \rangle=x^\top y$ denote the standard Euclidean inner product in $\R^d$, and $\|x\| = \sqrt{\langle x , x \rangle}$. Finally, recall that  $N(x) \coloneqq x / \|x\|$.

For any collection of tokens $\{x_1, \ldots, x_n\}$ in $\R^d$, define $y_i=N(x_i) \in \mathbb{S}^{d-1}$ for $i=1, \ldots, n$ and
    \begin{align}\label{e:Aij aij}
        Z_i \coloneqq \sum_{k=1} ^n e^{a_{ik}} \,, \qquad A_{ij} \coloneqq \frac{e^{a_{ij}}}{Z_i}\,, \qquad  a_{ij} \coloneqq \bea \left\langle y_i , y_j \right\rangle\,,
    \end{align}
for $i,j=1, \ldots, n$.
We then define
    \begin{align}\label{e:att}
        \att(y_i) \coloneqq \sum_{j=1} ^n A_{ij} y_j.
    \end{align}
Since the seminal work of~\cite{he2016deep}, residual connections are added to modern architectures and naturally act as a regularization scheme of the attention map towards the identity; see~\cite{chen2025residual}. With said residual connections, each token   $x_i$ is mapped to $x'_i$ using the following update rule
    \begin{align}\label{e:att update}
        x_i ' \coloneqq  \att(y_i)+  \ala x_i\,, \qquad \alpha \ge 0\,.
    \end{align}
Our first goal is to understand where the angle $\measuredangle(x_i', x_j')$ compares to  $\measuredangle(x_i, x_j)$. If $\measuredangle(x_i', x_j')<\measuredangle(x_i, x_j)$---or equivalently $\langle y_i',y_j'\rangle > \langle y_i, y_j\rangle$, with $y_i'=N(x_i')$---we say that attention is \emph{contractive}.

The nonlinear update rule (\ref{e:att update}) can produce complex dynamics, in which some pairs of tokens move closer together while others drift apart. This diversity of motion is in fact the most desirable outcome in practice, and it emerges precisely at the phase transition identified in this study. Beyond this critical regime, the tokens exhibit an unexpectedly cohesive behavior. To delineate the boundaries of the critical regime, we assume that the size and relative positions of the initial tokens are governed by constants independent of the number~$n$ of tokens. As an analytically tractable extreme of this assumption, we first consider the case in which the tokens form a regular simplex in~$\mathbb{R}^d$ as in~\cite{cowsik2024geometric}. Despite its symmetry, this configuration is sufficient to capture and predict the onset of the phase transition. We subsequently relax this constraint in Section~\ref{sec:equal angle} to show that the same phase transition occurs in more realistic configurations.

\subsection{The simplex case}\label{sec:intro simplex}

The following assumption was made in~\cite{cowsik2024geometric} and subsequently in \cite{giorlandino2025failuremodesdeeptransformers}. While rather stringent---in particular, it requires $d\ge n$---it turns out to provide a tractable yet predictive setup to study the contractive properties of attention.

\begin{assumption}\label{a:simplex}
There exists nonnegative constants $q\ge 0$ and $\rho \in (0,1)$ such that  $\|x_i\|^2  = q$ and $\langle y_i, y_j \rangle= \rho$, for any $i, j=1, \ldots, n$ and $i \neq j$.
\end{assumption}

Under Assumption \ref{a:simplex}, it is easy to see that there are positive constants $\rho'$ and $q'$ such that $\langle y_i',y_j'\rangle=\rho'$ for all $i\neq j$ and $\|x_i'\|^2= q'$ for all $i$. This simplification gives rise to a tractable phase transition.

\begin{theorem}\label{thm:att limit phase 2}
Under Assumption \ref{a:simplex}, there is a $\rho' \in (0,1)$ such that $\langle y_i ' , y_j '\rangle=\rho'$ for all $i\neq j$. Moreover, if $\bea = \gamma \log n$ where $\gamma$ is a positive constant, then for any  $i\neq j$, it holds 
    \begin{align}\label{e:simplex_phase}
        \lim_{n \to +\infty} \langle y_i ' , y_j ' \rangle = 
        \begin{cases}
            \frac{\rho (\ala\sqrt{q}+1)^2}{\ala^2 q+ 2\ala \sqrt{q}\rho + \rho} & \text{if $\gamma < \frac{1}{1-\rho}$},
            \\  \frac{ \rho (\ala\sqrt{q}+1)^2}{\ala^2 q + \ala \sqrt{q} (1+\rho) + \frac{1+3\rho}{4}} & \text{if $\gamma = \frac{1}{1-\rho}$},
            \\  \rho & \text{if $\gamma > \frac{1}{1-\rho}$}.
        \end{cases}
    \end{align}
\end{theorem}
Note that when $\gamma \leq \frac{1}{1-\rho}$, the right hand sides of Equation~(\ref{e:simplex_phase}) are strictly larger than~$\rho$ for any $\alpha\ge 0$. In other words, in the critical and subcritical regimes attention is contractive even in the presence of a residual connection. Of course, when $\alpha \to \infty$, the effects of attention dissipates and the limit tends to $\rho$ for all phases. This is expected as the update from $y_i$ to $y_i'$ tends to the identity map, an effect known to mitigate oversmoothing'' in residual neural networks; see~\cite{chen2025residual}.

Note also that for $\alpha=0$, that is in absence of residual connections, the limit in Equation~(\ref{e:simplex_phase}) reduces to
    \begin{align}\label{e:simplex_phase_alpha=0}
        \lim_{n \to +\infty} \langle y_i ' , y_j ' \rangle = 
        \begin{cases}
           1 & \text{if $\gamma < \frac{1}{1-\rho}$},
            \\  \frac{ 4\rho }{1+3\rho} & \text{if $\gamma = \frac{1}{1-\rho}$},
            \\  \rho & \text{if $\gamma > \frac{1}{1-\rho}$}.
        \end{cases}
    \end{align}

In the subcritical case, the tokens contract in one step towards a single cluster when $n \to \infty$ while in the supercritical case, their inner product does not change. In fact, a careful inspection of the proof reveals that in this supercritical regime the attention operator converges to the identity as $n \to \infty$. When $\alpha>0$, the subcritical case is mitigated by the residual connection which prevents token to collapse to a single point in one step. Nevertheless, this singular behavior reveals a major limitation in the simplex case: since the tokens are equidistant the phase transition reveals an all-or-nothing phenomenon where attention transitions from  $A_{ij}\sim 1/n$ so that $\att(y_i)=\bar y = \frac{1}{n} \sum_{j=1} ^n y_j$ for all $i$ to $A_{ij}=\delta_{ij}$ so that $\att(y_i)=y_i$ for all $i$. In the next section, we present a similar result Theorem~\ref{thm:att limit phase}, where the simplex assumption is relaxed. 

Before we end this section, we present the proof for Equation~(\ref{e:simplex_phase_alpha=0}) as a special case of Theorem~\ref{thm:att limit phase 2}. The detailed proof for Theorem~\ref{thm:att limit phase 2} and the later Theorem~\ref{thm:att limit phase} in Section~\ref{sec:equal angle} is included in Appendix~\ref{sec:proof of angle phase transition}.

\begin{proof}[Proof of Equation~(\ref{e:simplex_phase_alpha=0})]
     In Equation~(\ref{e:att update}), when $\alpha = 0$, we have that $x_i ' = \att(y_i)$ for each $i=1,2,\dots,n$. In Equation~(\ref{e:Aij aij}), under Assumption \ref{a:simplex},  we notice that the quantity $\sum_{k=1} ^n e^{a_{ik}}$ in the denominator of $A_{ij}$ is independent of the choice of $i$, and equals to $e^{\bea} + (n-1) e^{\rho \bea }$. Denote this as $Z \coloneqq e^{\bea} + (n-1) e^{\rho \bea }$. Then Equation~(\ref{e:att}) and (\ref{e:att update}) become
        \begin{align*}
             x_i ' = \att(y_i) = \frac{1}{Z} \left(e^{\bea} y_i + \sum_{m \neq i} e^{\rho \bea} y_m\right).
        \end{align*}
    Under Assumption \ref{a:simplex}, a direct computation shows that for any $i=1,2,\dots,n$,
        \begin{align*}
            \langle x_i ' , x_i ' \rangle =\frac{1}{Z^2}  \left( e^{2\bea} + 2 (n-1)\rho e^{ (1+\rho)\bea} + (n-1)(1+(n-2)\rho) e^{2 \rho \bea} \right),
        \end{align*}
    and for any two different $i,j=1,2,\dots,n$,
        \begin{align*}
            \langle x_i ' , x_j ' \rangle = \frac{1}{Z^2} \left(\rho e^{2\bea}+ 2(1+(n-2)\rho)e^{\bea(1+\rho)} + \left((n-2)+(n^2-3n+3)\rho \right) e^{2\bea \rho} \right).
        \end{align*}
    See also Lemma~\ref{lem:att update length} and Lemma~\ref{lem:att update angle} for more detailed computations for $\langle x_i ' , x_i ' \rangle$ and $\langle x_i ' , x_j ' \rangle$.

    For $Z = e^{\bea} + (n-1) e^{\rho \bea }$, when we let $\beta = \gamma \log n$, we see that $e^{\bea} = n^{\gamma}$ and $n e^{\rho \bea} = n^{1+  \rho \gamma}$ in $Z$. The largest term in $Z$ then depends on the relation between $\gamma$ and $1+  \rho \gamma$: when $\gamma < \frac{1}{1-\rho}$, $n^{1+  \rho \gamma}$ is the largest term; when $\gamma > \frac{1}{1-\rho}$, $n^{\gamma}$ is the largest term. We then directly get the following three phases for $Z$ from the above arguments:
    \begin{align}\label{e:Z three phase}
        Z = 
        \begin{cases}
           (1+o_n(1)) \cdot n e^{\rho \bea }   & \text{if $\gamma < \frac{1}{1-\rho}$},
           \\ (2+o_n(1)) \cdot e^{\bea}  & \text{if $\gamma = \frac{1}{1-\rho}$},
            \\ (1+o_n(1)) \cdot e^{\bea}  & \text{if $\gamma > \frac{1}{1-\rho}$},
        \end{cases}
    \end{align}
    where the terms $o_n(1)$ go to $0$ as $n \to +\infty$.
    Similarly, we can get the following three phases for $\langle x_i ' , x_i ' \rangle$:
        \begin{align}
        \lim_{n \to +\infty} \langle x_i ' , x_i ' \rangle = 
        \begin{cases}
             \rho  & \text{if $\gamma < \frac{1}{1-\rho}$},
            \\  \frac{1+3\rho}{4} & \text{if $\gamma = \frac{1}{1-\rho}$},
            \\ 1 & \text{if $\gamma > \frac{1}{1-\rho}$}.
        \end{cases}
    \end{align}
    For $\langle x_i ' , x_j ' \rangle$, we always have that $\lim_{n \to +\infty} \langle x_i ' , x_j ' \rangle = \rho$ for $\gamma$ in these three different regimes. Then Equation~(\ref{e:simplex_phase_alpha=0}) follows from these two limits because $\langle y_i ' , y_j ' \rangle = \langle x_i ' / \| x_i ' \| , x_j ' / \| x_j ' \| \rangle$.
\end{proof}

\subsection{The almost-simplex case}\label{sec:equal angle}

In this section, we relax Assumption~\ref{a:simplex} to allow pairwise angles and lengths to vary slightly. This relaxation  makes it possible for tokens to lie in a dimension $d \ll n$. Although the resulting bounds are not as sharp as those obtained under Assumption~\ref{a:simplex}, they demonstrate that the critical scaling $\beta_n = \Theta(\log n)$ is intrinsic and not merely an artifact of a particular geometric construction.

\begin{assumption}\label{a:equal angle}
There exist constants $q_1, q_2 \in (0,\infty), \rho_1, \rho_2 \in (0,1)$ such that  $q_1 \le \|x_i\|^2 \le q_2$ and $\rho_1 \leq  \langle y_i, y_j \rangle\leq \rho_2$, for any $i, j=1, \ldots, n$ and $i \neq j$. Moreover, $\rho_1 =  \langle y_i, y_j \rangle$ for some $i,j$. 
\end{assumption}

\begin{remark}
    Assumption~\ref{a:equal angle} already allows near-ties when $\rho_2$ is close to $1$, and it can be further generalized to allow multiple exact ties in the top scores. Specifically, one may assume that there exists a fixed $k \in \mathbb{Z}_{+}$ such that, for each $i \in \{1,\dots,n\}$, there are at most $k$ indices $j$ with $\langle y_i, y_j \rangle = 1$. Under this setting, all of our main results continue to hold with the same critical scaling order $\log n$ but with different constants depending on $k$.
    In fact, this setting is a special case of the more general setting discussed in Appendix~\ref{sec:more transition phases}, where Assumption~\ref{a:three phase} partitions the inner products into three ranges, $[\rho_1,\rho_2]$, $[\rho_3,\rho_4]$, and $\{1\}$, with $0 \leq \rho_1 < \rho_2 < \rho_3 < \rho_4 <1$. The phase transition behavior remains of order $\log n$ in that general setting as well. For clarity and readability of the main exposition, we keep Assumption~\ref{a:equal angle} in the main text.
\end{remark}

It is easy to see using standard probabilistic tools that Assumption~\ref{a:equal angle} holds with high probability when the $y_i$'s are independent random vectors  uniformly distributed on a half-sphere for example.

\begin{theorem}\label{thm:att limit phase}
Under Assumption \ref{a:equal angle}, we have the following phase transition when $\bea = \gamma \log n$ for some fixed $\gamma>0$.

If $\gamma < \frac{1}{1-\rho_1}$, then there is a constant $\varepsilon>0$ depending on $\ala,\rho_2,q_1,q_2$, such that
    \begin{align}\label{e:update phase 1}
        \varliminf_{n \to +\infty} \min_{i \neq j}\langle y_i ' , y_j ' \rangle \geq \rho_1+\varepsilon >\rho_1,
    \end{align}
which implies that the angle between tokens becomes strictly smaller after an attention layer Equation~(\ref{e:att update}).

If $\gamma > \frac{1}{1-\rho_2}$, then for any $i \in \llbracket 1 , n \rrbracket$,
        \begin{align}\label{e:update phase 2 identity}
        \att(y_i) = y_i + o_n(1), \text{ and hence } x_i '  = y_i + \ala x_i + o_n(1),
        \end{align}
    where the term $o_n(1)$ goes to $0$ as $n \to +\infty$ with a speed uniform in $i$. Hence, when $\gamma > \frac{1}{1-\rho_2}$, for any two different $i,j \in \llbracket 1 , n \rrbracket$,
    \begin{align}\label{e:update phase 2}
        \lim_{n \to +\infty} \langle y_i ' , y_j ' \rangle = \langle y_i , y_j \rangle.
    \end{align}
which implies that the angle between tokens does not change after an attention layer Equation~(\ref{e:att update}).
\end{theorem}

The proof for Theorem~\ref{thm:att limit phase} is included in Appendix~\ref{sec:proof of angle phase transition}, but the general intuition is similar to the proof for Equation~(\ref{e:simplex_phase_alpha=0}) in Section~\ref{sec:intro simplex}. As we have seen in that proof, the first step to build up phase transition regimes for $\langle y_i ' , y_j ' \rangle$ is to study the phase transition regimes for $Z_i$ in Equation~(\ref{e:Aij aij}). Adjusting the logarithmic scaling factor $\gamma$ causes different phase transition regimes for $Z_i$ first. When $\gamma$ is small enough, the weights $e^{a_{ik}}$ consisting of $Z_i$ are asymptotically uniform, and each token almost equally interacts with the other tokens. When $\gamma$ is large enough, each token mostly focuses on itself. 

Building on this observation, Theorem~\ref{thm:att limit phase 2} and Theorem~\ref{thm:att limit phase} together demonstrate that $\gamma$ controls the effective interaction range of each token. In particular, we have seen in Theorem~\ref{thm:att limit phase 2} the existence of the critical regime when $\gamma = \frac{1}{1-\rho}$. In this case, although the tokens continue to contract, their rate of shrinkage is evidently slower than in the subcritical regime, as shown in Equation~(\ref{e:simplex_phase}) and Equation~(\ref{e:simplex_phase_alpha=0}). 

It is hence natural to ask whether further regimes emerge when $\gamma$ is varied between the supercritical and subcritical threshold. Indeed, in Appendix~\ref{sec:more transition phases}, we prove the existence of a nontrivial middle phase when $\gamma$ is between the two extrema $\frac{1}{1-\rho_1}$ and $\frac{1}{1-\rho_2}$, under a refined assumption on the distribution of tokens, which allows for a sharper characterization of the transition. Under this refined assumption, Theorem~\ref{thm:att three phase} show the existence of $\gamma_1, \gamma_2$ such that Equation~(\ref{e:att update}) presents three different phases: $\gamma < \gamma_1$, $\gamma_1 <\gamma < \gamma_2$, and $\gamma > \gamma_2$. In the extreme regimes, when $\gamma < \gamma_1$, each token interacts with almost all the remaining tokens, while when $\gamma > \gamma_2$, each token only focuses on itself, consistent with Theorem~\ref{thm:att limit phase}. In the intermediate regime $\gamma_1 < \gamma < \gamma_2$, however, the weights $e^{a_{ik}}$ concentrate on only a small subset of tokens, so that each $Z_i$ and hence the update in Equation~(\ref{e:att update}) is dominated by a few highly relevant interactions. This shows that the logarithmic scaling enables each token to dynamically select its most relevant context.

We conclude by noting that those $o_n(1)$ terms in our theorems satisfy the bound $|o_n(1)| \leq C_1 n^{-C_2}$ for some positive constants $C_1,C_2$ that are independent of $d$ (though varying across theorems). As a result, the simplex configuration (Assumption~\ref{a:simplex}) and the almost simplex configuration (Assumption~\ref{a:equal angle}) remains valid under repeated application of the $\att$ operator up to $\mathrm{poly}(n)$ iterations. In particular, the accumulated error remains negligible at this scale, so our theorems and arguments extend to transformers with many layers.

\subsection{Propagation of Gradients under Attention Layer}\label{sec:propagation results}

In the previous section, we established how attention scaling influences the propagation of token representations, corresponding to running the Transformer in the forward (inference) direction. During training, however, the Transformer is also executed in the \emph{backward} direction to compute gradients via backpropagation~\citep{rumelhart1986learning}. In this section, we show that a similar phase transition arises in the backward pass: in the subcritical regime---where token representations rapidly collapse in the forward pass---the gradients also collapse, whereas in the supercritical regime they retain their scale. The stability of gradients is a crucial computational consideration that strongly affects a model's ability to be trained effectively. For this reason, several theoretical analyses of gradient dynamics in Transformers have been conducted, albeit without attention scaling; see, for example,~\cite{cowsik2024geometric,dong2021attention,noci2022signal}.


Let the input token configuration be denoted by $X(0)$, and let $X(t)$ represent the positions of all tokens at the output of Transformer layer $t$. To compute gradients, one needs to evaluate the end-to-end input–output Jacobian across $L$ layers of the Transformer. By the chain rule, this Jacobian can be expressed as
\begin{align*}
    \frac{\partial X(L)}{\partial X(0)} 
    = 
    \frac{\partial X(L)}{\partial X(L-1)}
    \frac{\partial X(L-1)}{\partial X(L-2)} 
    \cdots 
    \frac{\partial X(1)}{\partial X(0)}.
\end{align*}
Thus, the end-to-end Jacobian can be obtained by recursively computing and multiplying the layer-wise Jacobians. This procedure is known as the \emph{adjoint method} in dynamical systems theory~\citep{lions1971optimal}, and as \emph{backpropagation} in the machine learning community.

Our main result shows that when $\beta_n = \gamma \log n$ with subcritical $\gamma$, the typical singular values of $\frac{\partial X(t+1)}{\partial X(t)}$ are close to zero (apart from the contribution of the residual connection). In contrast, for supercritical values of $\gamma$, the contribution of the attention component to the Jacobian is non-trivial and behaves as a normalization map.

We now proceed with formal definitions. For $x \in \R^d$, let $(x)_u$ denote its $u$-th coordinate for $u = 1, 2, \dots, d$. The concatenation $X = (x_1, x_2, \dots, x_n) \in \R^{nd}$ represents the configuration of all tokens. The normalization map is defined by
\begin{align}\label{e:def norm map}
    \cln(X)
    = \cln(x_1, x_2, \dots, x_n)
    \coloneqq \big(N(x_1), N(x_2), \dots, N(x_n)\big),
\end{align}
and the attention map by
\begin{align}\label{e:def att map}
    \clatt(Y)
    = \clatt(y_1, y_2, \dots, y_n)
    \coloneqq \big(\att(y_1), \att(y_2), \dots, \att(y_n)\big),
\end{align}
where $\att(y_i)$ is defined in~\eqref{e:att} and $Y = (y_1, \ldots, y_n)$.  
Under these definitions, the update~\eqref{e:att update} can be written compactly as
\begin{align}\label{e:att update matrix form}
    X' = \clatt(\cln(X)) + \alpha X,
\end{align}
where $X' = (x_1', x_2', \dots, x_n')$.  

\medskip
We define the $nd \times nd$ Jacobian matrix as
\begin{align}
    \nabla_X X' \coloneqq 
    \left( 
        \frac{\partial (x_j')_v}{\partial (x_i)_u}
    \right)_{(j,v), (i,u)},
\end{align}
for $i, j = 1, \ldots, n$ and $u, v = 1, \ldots, d$.  
The matrix norm of $\nabla_X X'$ is given by
\begin{align}\label{e:def jacobian norm}
    \|\nabla_X X'\|^2 
    \coloneqq  
    \mathrm{tr}\!\left[(\nabla_X X')^\top \nabla_X X'\right]
    = \sum_{i,j=1}^n \sum_{u,v=1}^d 
      \left(
        \frac{\partial (x_j')_v}{\partial (x_i)_u}
      \right)^2.
\end{align} 

Let $\sigma_1, \sigma_2, \dots, \sigma_{nd}$ denote the singular values of $\nabla_X X'$. Then the normalized Jacobian norm satisfies
\begin{align}\label{e:mean jacobian norm}
    \frac{1}{nd}\|\nabla_X X'\|^2 
    = \frac{1}{nd}\sum_{i=1}^{nd} \sigma_i^2,
\end{align}
which represents the mean squared singular value of the Jacobian.  

\medskip
Before stating our results on $\frac{1}{nd}\|\nabla_X X'\|^2$, we note that the Jacobian $\nabla_X X'$ can be decomposed into the residual part $\alpha I_{nd}$ and the attention part $\nabla_X\!\big(\clatt(\cln(X))\big)$.  
As shown in Theorems~\ref{thm:att limit phase 2} and~\ref{thm:att limit phase}, the residual component $\alpha I_{nd}$ does not affect the phase transition behavior.  
Therefore, to streamline the analysis, we focus exclusively on the attention term $\nabla_X\!\big(\clatt(\cln(X))\big)$ by setting $\alpha = 0$ in~\eqref{e:att update matrix form}.  
The following theorems characterize $\frac{1}{nd}\|\nabla_X X'\|^2$ under this setting.

\begin{theorem}\label{thm:att propagation gradient norm 2}
Adopt Assumption~\ref{a:simplex} and Equation~(\ref{e:att update matrix form}) with $\alpha = 0$. Then, we have the following phase transition phenomenon: let $\bea = \gamma \log n$ where $\gamma$ is a positive constant.
    
    If $\gamma < \frac{1}{1-\rho}$,
        \begin{align}\label{e:propagation gradient phase 1 simplex}
            \frac{1}{nd}\|\nabla_X X'\|^2=  0+o_n(1).
        \end{align}

    If $\gamma = \frac{1}{1-\rho}$
        \begin{align}\label{e:propagation gradient phase middle simplex}
            \frac{1}{nd}\|\nabla_X X'\|^2=  \frac{1}{4q} \left(1-\frac{1}{d} \right)+o_n(1).
        \end{align}

    If $\gamma > \frac{1}{1-\rho}$
        \begin{align}\label{e:propagation gradient phase 2 simplex}
            \frac{1}{nd}\|\nabla_X X'\|^2 = \frac{1}{q} \left(1-\frac{1}{d} \right) + o_n(1).
        \end{align} 

    In both cases, the terms $o_n(1)$ go to $0$ as $n \to +\infty$, with speeds depending on $\gamma,\rho,q$.
\end{theorem}

The results of the previous theorem show that under the simplex assumption, the phase transition in the backward dynamics (for gradients) is as sharp as for the forward pass: for small $\gamma$, gradients do not flow through the attention block.

We can also extend the analysis for Theorem~\ref{thm:att propagation gradient norm 2} to the relaxed Assumption~\ref{a:equal angle}.

\begin{theorem}\label{thm:att propagation gradient norm}
Adopt Assumption \ref{a:equal angle} and Equation~(\ref{e:att update matrix form}) with $\alpha = 0$. Then, we have the following phase transition phenomenon: let $\beta = \gamma \log n$ where $\gamma$ is a positive constant.
    
    If $\gamma < \frac{1}{1-\rho_1}$,
        \begin{align}\label{e:propagation gradient phase 1}
             \frac{1}{nd} \|\nabla_X X'\|^2 \leq 4\frac{\gamma^2 (\log(n))^2}{q_1d} + o_n(1),
        \end{align}

    If $\gamma > \frac{1}{1-\rho_2}$, 
        \begin{align}
            \frac{1}{nd} \|\nabla_X X'\|^2 \geq \frac{1}{q_2} \left(1- \frac{1}{d}\right) + o_n(1),
        \end{align}
    which is away from $0$ even when $d,n$ is very large. Indeed, when $\gamma > \frac{1}{1-\rho_2}$, for any fixed $i,j \in \llbracket 1, n \rrbracket$,
        \begin{align}\label{e:propagation gradient phase 2}
             \left(\frac{\partial (\att(N(x_j)))_v}{\partial (x_i)_u} \right)_{d \times d} = \frac{\delta_{ij}}{\|x_i\|} \left( I_d - y_i  y_i ^T \right) + \bfo_n(1) + o_n(1) \cdot  I_d,
        \end{align} 
    where the leading order term is exactly $ \frac{\partial (N(x_j))_v}{\partial (x_i)_u}$ as shown in Proposition \ref{lem:norm gradient}. Here, $I_d$ is the $d \times d$ identity matrix, the term $\bfo_n(1)$ ($o_n(1)$, respectively) is a $d \times d $ matrix (constant, respectively) with matrix norm as defined in Equation~(\ref{e:def jacobian norm}) (value, respectively) going to $0$ as $n \to +\infty$, with a speed independent of $i,j$ but only depending on $\gamma,\rho_2,q_1$.
\end{theorem}

We present the proofs for Theorem~\ref{thm:att propagation gradient norm 2} and Theorem~\ref{thm:att propagation gradient norm} in Appendix~\ref{sec:proof of gradient phase transition}. Note that the $\frac{\log^2 n}{d}$ term in \eqref{e:propagation gradient phase 1} is small for typical values of $n$ and $d$ used in Transformers. Theorem~\ref{thm:att propagation gradient norm 2} and Theorem~\ref{thm:att propagation gradient norm} also corroborate the fact that tokens collapse fast when $\gamma$ is in the subcritical regime, while each token only focuses on itself when $\gamma$ is in the supercritical regime.

\section{Numerical Experiments}\label{sec:num}

This section reports numerical experiments designed to support our theoretical predictions. 
In the following numerical experiments, we test the phase transition in the almost-simplex case as Section~\ref{sec:equal angle}. We generate samples $\{x_1,\ldots,x_n\}\subset\mathbb{R}^d$ such that the expectations $\mathbb{E}\|x_i\|^2=1$ and $\mathbb{E}\langle x_i, x_j \rangle = \rho\in[0,1]$ for $i\neq j$.  More precisely, we generate $x_i$ according to
\begin{align}
x_i = \sqrt{\rho}\,z_0 + \sqrt{1-\rho}\,z_i \,,
\end{align}
where $z_0, z_1, \ldots, z_n$ are i.i.d. standard Gaussian vectors in $\mathbb{R}^d$. The generated samples satisfy the Assumption \ref{a:equal angle} with high probability.

In Figure~\ref{fig:lambda_a_single_attn}, we plot the input-to-output angle ratio $\lambda$, defined as
\begin{align}\label{eqn:input_output_ratio}
    \lambda = \frac{2}{n(n-1)} \sum_{1\leq i<j\leq n} \frac{1-\langle y_i', y_j' \rangle}{1-\langle y_i , y_j \rangle} \,,
\end{align}
for samples processed through a single self-attention layer with different $\gamma$ and of different dimensions $d$. Consistent with our theoretical predictions, the layer acts as a contraction mapping when $\gamma$ is small, reducing pairwise output angles, whereas for large $\gamma$ the output angles remain nearly unchanged from the input. Moreover, in the large $d$ regime the angle between input tokens $\langle y_i, y_j \rangle$ ($i\neq j$) concentrate near $\rho$, so that the simplex Assumption~\ref{a:simplex} is effectively satisfied. In this setting, we observe a sharp phase transition in agreement with Theorem~\ref{thm:att limit phase 2}. In the small $d$ regime, however, the input tokens $\langle y_i, y_j \rangle$ randomly distributed in an interval $(\rho_1,\rho_2)$, and an intermediate phase emerges in which the contraction is only partial: some angles shrink significantly while others remain close to their original values, which smooths out the transition.

\begin{figure}[h!]
\centering \hspace*{-0.9cm}
\begin{subfigure}{0.32\textwidth}
  \centering
  \includegraphics[scale=0.345]{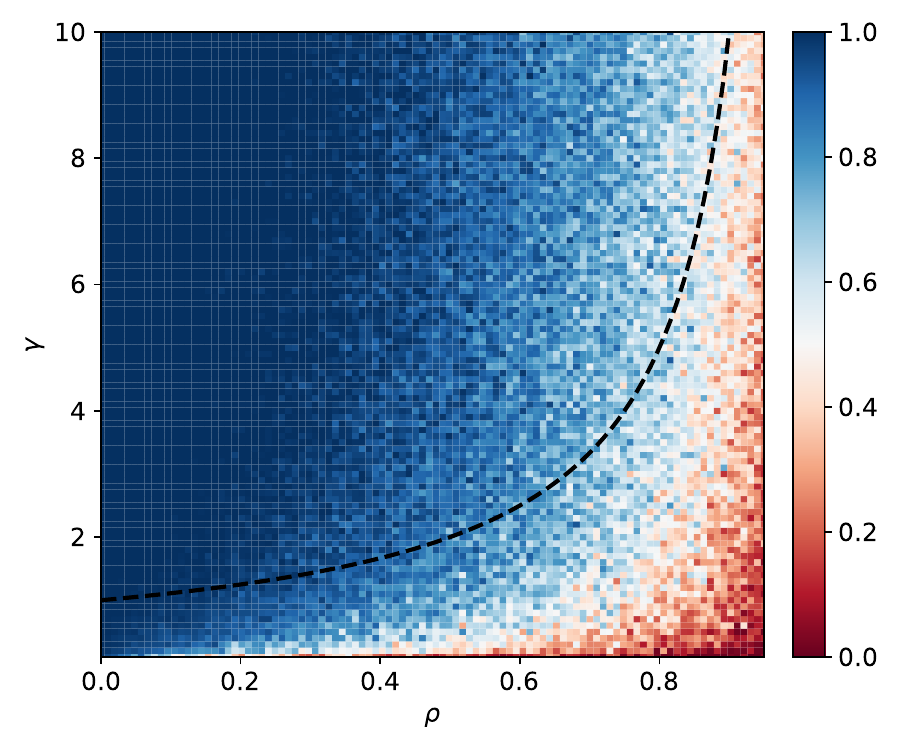}
  \caption{$d = 2$}
\end{subfigure}%
\begin{subfigure}{0.32\textwidth}
  \centering
  \includegraphics[scale=0.345]{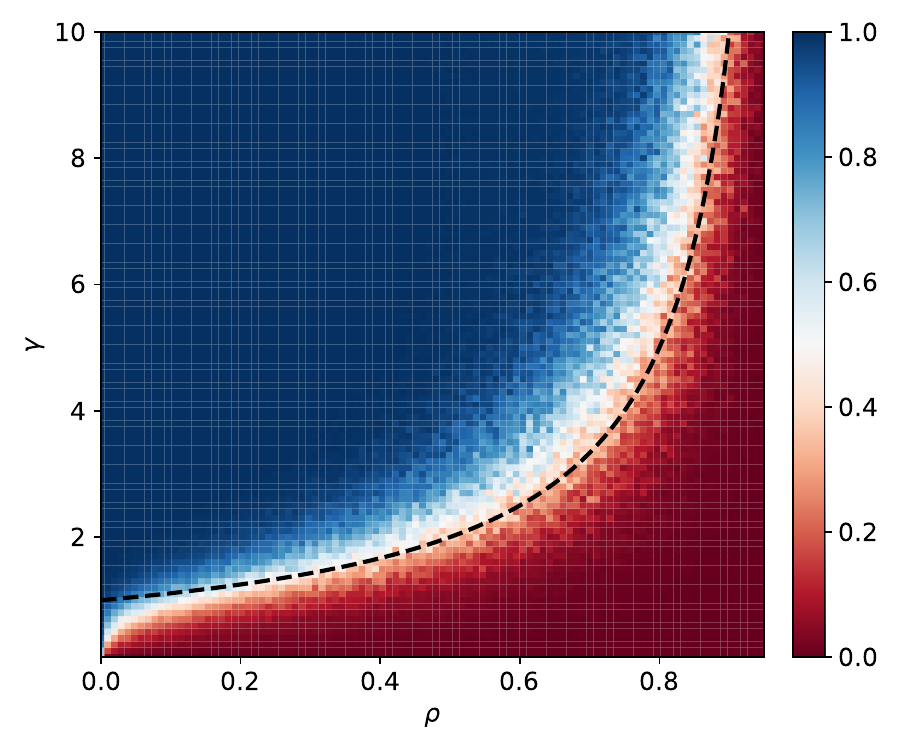}
  \caption{$d = 32$}
\end{subfigure}%
\begin{subfigure}{0.32\textwidth}
  \centering
  \includegraphics[scale=0.345]{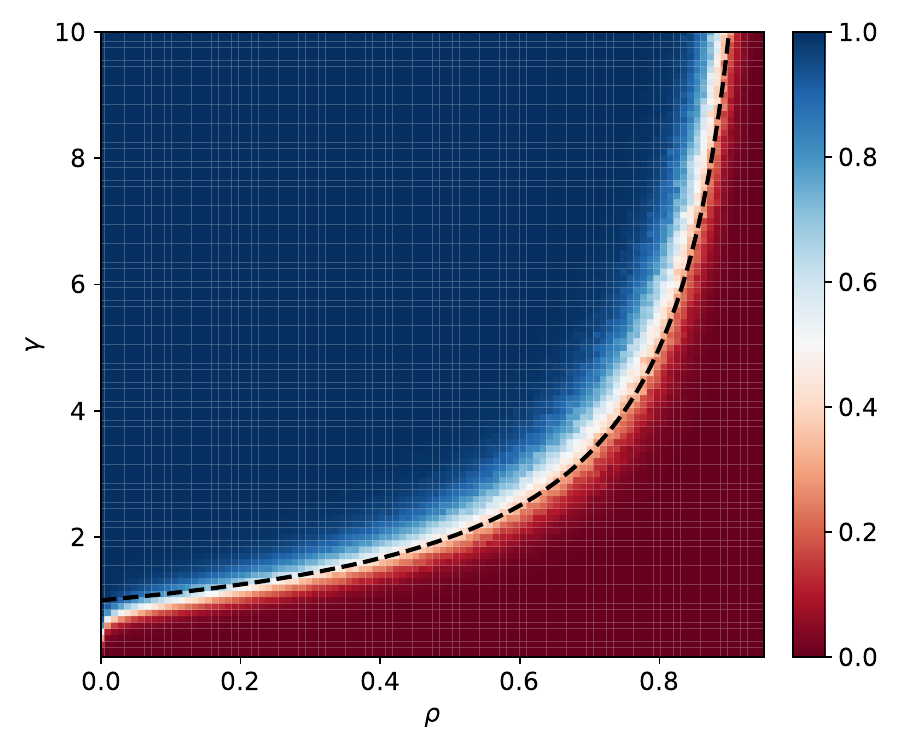}
  \caption{$d = 512$}
\end{subfigure}
\caption{
Plots of the input-to-output angle ratio $\lambda$, defined in Equation~(\ref{eqn:input_output_ratio}), as a function of $\rho$ and $\gamma$. The tokens are first normalized by a pre-layer normalization and then passed through a single self-attention layer (\ref{e:att}), with residual connections and MLP layers omitted. The dashed curve corresponds to $\gamma=\tfrac{1}{1-\rho}$, which approximates the actual phase transition with increasing accuracy as $d$ grows, as implied by Theorem \ref{thm:att limit phase 2}.
}\label{fig:lambda_a_single_attn}
\end{figure}

In Figure \ref{fig:lambda_g_single_attn}, we plot the normalized matrix norm for the $nd \times nd $ matrix $\nabla_X X'$, defined as
\begin{align}\label{eqn:normalized_norm}
    \eta = \frac{1}{nd} \|\nabla_X X'\|^2 \,,
\end{align}
for samples passed through a single self-attention layer with varying $\gamma$ and dimension $d$. Across all three plots, the normalized gradient norm remains close to $0$ when $\gamma$ is small, while for large $\gamma$ it approaches $1-1/d$, consistent with Theorem~\ref{thm:att propagation gradient norm}. Similar to the token-angle behavior, a sharp phase transition emerges near $\gamma=\tfrac{1}{1-\rho}$ in the large-$d$ regime, in agreement with the predictions under the simplex assumption. In lower dimensions, fluctuations in the pairwise angle prevent perfect concentration, and the transition is smoothed into an intermediate regime where the gradient norm only partially stabilizes.

\begin{figure}[h!]
\centering \hspace*{-0.9cm}
\begin{subfigure}{.32\textwidth}
  \centering
  \includegraphics[scale=0.345]{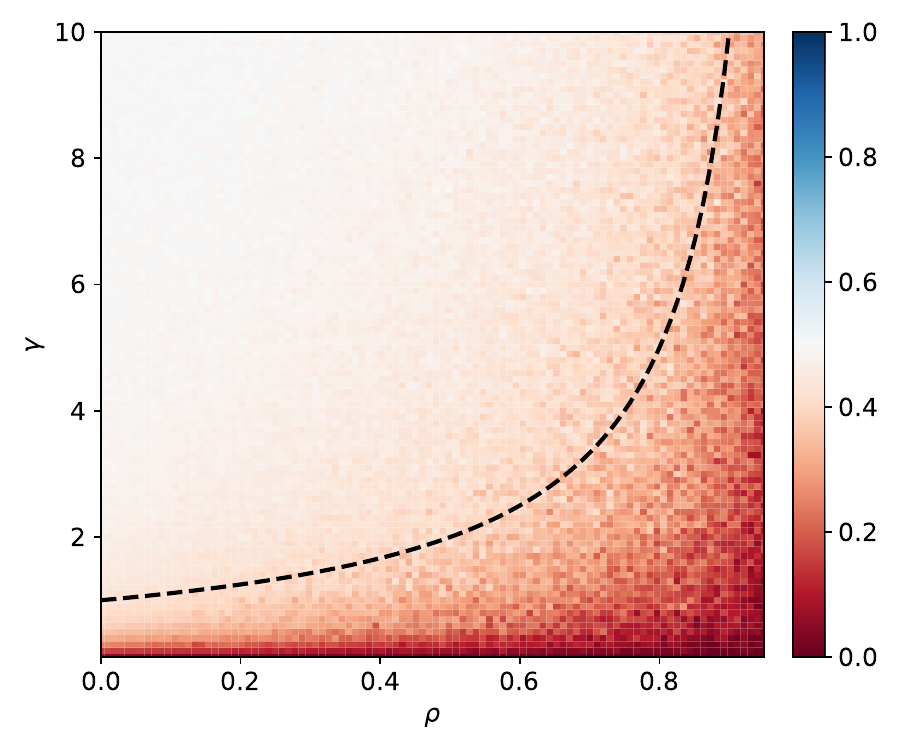}
  \caption{$d = 2$}
\end{subfigure}%
\begin{subfigure}{.32\textwidth}
  \centering
  \includegraphics[scale=0.345]{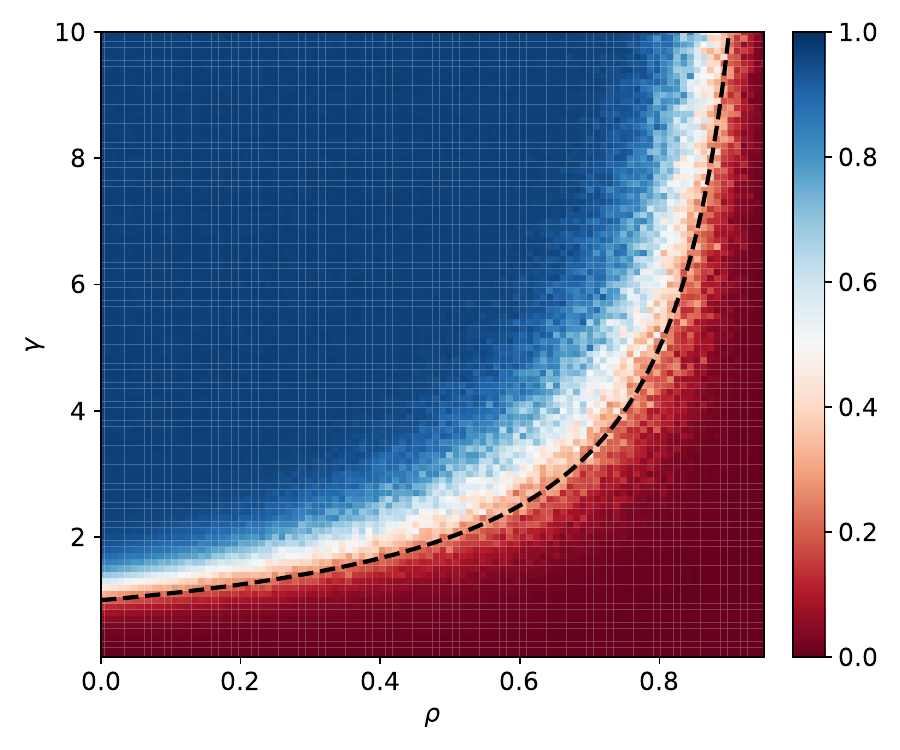}
  \caption{$d = 32$}
\end{subfigure}%
\begin{subfigure}{.32\textwidth}
  \centering
  \includegraphics[scale=0.345]{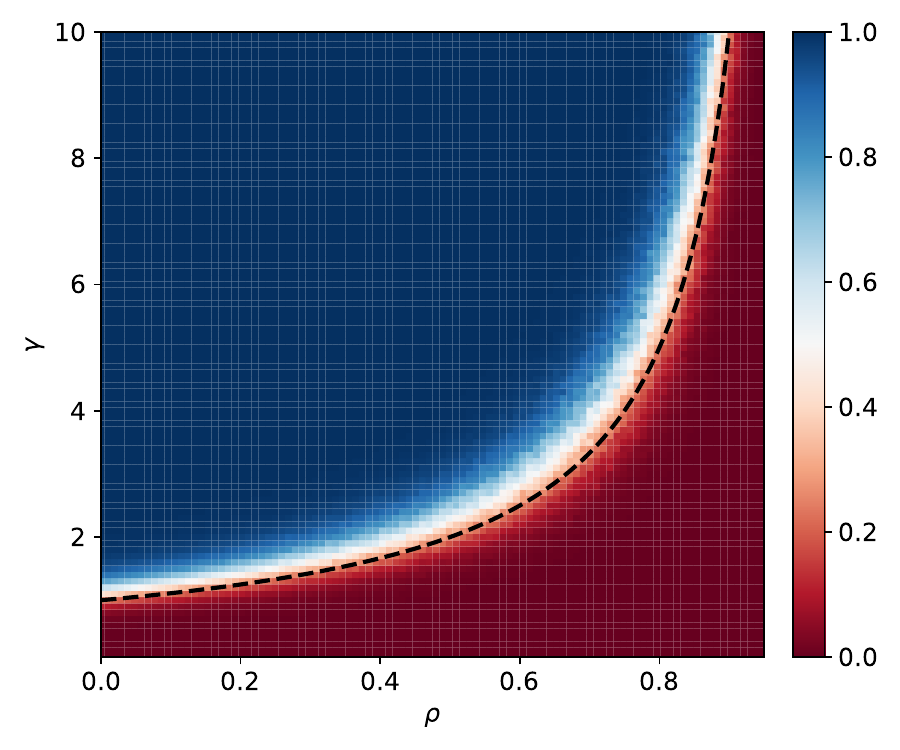}
  \caption{$d = 512$}
\end{subfigure}
\caption{
Plots of the normalized norm $\eta$ of the gradient, defined by Equation~(\ref{eqn:normalized_norm}), as a function of $\rho$ and $\gamma$. The tokens are first normalized by a pre-layer normalization and then passed through a single self-attention layer (\ref{e:att}), with residual connections and MLP layers omitted. The dash curve shows $\frac{1}{1-\rho}$, which approximate the actual phase transition with increasing accuracy as $d$ grows, as implied by Theorem \ref{thm:att propagation gradient norm 2}. The matrix norm $\eta$ is computed by the Hutchinson trace estimator~\citep{hutchinson1989stochastic}, based on the definition in Equation~(\ref{e:def jacobian norm}).
}\label{fig:lambda_g_single_attn}
\end{figure}

\section{Conclusion}

This paper develops a framework for understanding phase transitions in self-attention as the context length $n$ grows, identifying a critical scaling $\beta_n \asymp \log n$ that separates a subcritical contractive regime from a supercritical unchanged regime. A central message is that this transition is rooted in the geometry of the score landscape: in our model, the gaps between the top ordered scores remain $O(1)$, which leads to the $\log n$ scaling. We show that this scaling is robust under various perturbations and in settings permitting multiple near-ties, demonstrating that the $\log n$ law is a structural consequence of content-adaptive interactions.

\subsubsection*{Acknowledgments}
Philippe Rigollet is supported by NSF grants DMS-2022448.

\subsubsection*{LLM Usage}

Large Language Models (LLMs) were used during peer review for grammar and syntax refinement only. All ideas, technical content, analyses and conclusions remain the authors' work.



\appendix

\section{Proof of Theorem~\ref{thm:att limit phase 2} and Theorem~\ref{thm:att limit phase}}\label{sec:proof of angle phase transition}
In this section, we adopt Assumption~\ref{a:equal angle} and prove Theorem~\ref{thm:att limit phase} first. Then, we prove Theorem~\ref{thm:att limit phase 2}. To simplify notations, we define $\llbracket 1 , n \rrbracket \coloneqq \{1,2,\dots,n\}$ for any $n \in \mathbb{Z}_+$.

We study the asymptotics of the quantity $\langle x_i ' , x_j ' \rangle$ as $n \to +\infty$. We use the notation
    \begin{align}\label{e:def Zi}
        Z_i \coloneqq \sum_{k=1} ^n e^{a_{ik}} = e^{\bea} + \sum_{k \neq i} e^{a_{ik}}.
    \end{align}

\begin{lemma}\label{lem:Z_i asymptotics}
    Let $\bea = \gamma \log n$ where $\gamma$ is a positive constant. Under Assumption \ref{a:equal angle} and Equation~(\ref{e:att update}), for any $i \in \llbracket 1 , n \rrbracket$,
    \begin{align}\label{e:Z_i phase}
        Z_i = 
        \begin{cases}
           (1+o_n(1)) \cdot \left(\sum_{k \neq i} e^{a_{ik}}\right)   & \text{if $\gamma < \frac{1}{1-\rho_1}$},
            \\ (1+o_n(1)) \cdot e^{\bea}  & \text{if $\gamma > \frac{1}{1-\rho_2}$},
        \end{cases}
    \end{align}
where the terms $o_n(1)$ go to $0$ as $n \to +\infty$ with speeds independent of $i$ but only depending on $\gamma,\rho_1, \rho_2$.
\end{lemma}
\begin{proof}[Proof of Lemma~\ref{lem:Z_i asymptotics}]
    We notice that 
        \begin{align}
            Z_i = e^{\bea} + \sum_{k \neq i} e^{a_{ik}}.
        \end{align}
    We also notice that $e^{\bea t} = n^{\gamma t}$ for any $t$. It then holds that $e^{\bea} = n^{\gamma}$ and 
        \begin{align}
n^{\gamma \rho_1} (n-1) \leq   \sum_{k \neq i} e^{a_{ik}} \leq n^{\gamma \rho_2} (n-1)\,.
        \end{align}
Hence, when $\gamma < \frac{1}{1-\rho_1}$, $n^{\gamma} < n^{1+\gamma\rho_1}$, the leading order term in $Z_i$ is $\sum_{k \neq i} e^{a_{ik}}$. We also see that
    \begin{align}
        Z_i = \left(\frac{e^{\bea}}{\left(\sum_{k \neq i} e^{a_{ik}}\right)} + 1\right) \cdot \left(\sum_{k \neq i} e^{a_{ik}}\right),
    \end{align}
with
    \begin{align}
        \frac{e^{\bea}}{\left(\sum_{k \neq i} e^{a_{ik}}\right)} \leq \frac{n^\gamma}{n^{\gamma \rho_1} (n-1)},
    \end{align}
which goes to $0$ as $n \to +\infty$, and is independent of $i$ but only depending on $\gamma,\rho_1$. Similarly, when $\gamma > \frac{1}{1-\rho_2}$, $n^{\gamma} > n^{1+\gamma\rho_2}$, the leading order term in $Z_i$ is $e^{\bea}$, and similar arguments hold true.
\end{proof}

\begin{lemma}\label{lem:gamma large single point}
    Let $\bea = \gamma \log n$ where $\gamma$ is a positive constant. Under Assumption \ref{a:equal angle} and Equation~(\ref{e:att update}), if $\gamma > \frac{1}{1-\rho_2}$, then for any $i \in \llbracket 1 , n \rrbracket$,
        \begin{align}
            \att(y_i) = y_i + o_n(1), \text{ and hence } x_i '  = y_i + \ala x_i + o_n(1),
        \end{align}
    where the term $o_n(1)$ goes to $0$ as $n \to +\infty$ with a speed independent of $i$ but only depending on $\gamma,\rho_2$.
\end{lemma}
\begin{proof}[Proof of Lemma~\ref{lem:gamma large single point}]
    According to Lemma \ref{lem:Z_i asymptotics}, we see that when $\gamma > \frac{1}{1-\rho_2}$, $n^{\gamma} > n^{1+\gamma\rho_2}$, and hence,
        \begin{align}
            \att(y_i) = Z_i ^{-1} \left(e^{\bea} y_i + \sum_{j \neq i} e^{a_{ij}} y_j\right) = (1+o_n(1))\left(y_i+e^{-\bea} \sum_{j \neq i} e^{a_{ij}} y_j \right).
        \end{align}
    Because $\|y_j\|=1$, 
        \begin{align}
            \left\| e^{-\bea} \sum_{j \neq i} e^{a_{ij}} y_j \right\| \leq e^{-\bea} \sum_{j \neq i} e^{a_{ij}} \leq n^{-\gamma} \cdot n^{\gamma \rho_2} (n-1),
        \end{align}
which goes to $0$ as $n \to +\infty$, and is independent of $i$ but only depending on $\gamma,\rho_2$.  
This shows that when $\gamma > \frac{1}{1-\rho_2}$,
    \begin{align}
         \att(y_i) = (1+o_n(1))(y_i +o_n(1)) = y_i + o_n(1).
    \end{align}
\end{proof}


\begin{lemma}\label{lem:att update length}
    Under Assumption \ref{a:equal angle} and Equation~(\ref{e:att update}), for any $i \in \llbracket 1 , n \rrbracket$,
    \begin{align}\label{e:att update length}
        \begin{split}
            \langle x_i ' , x_i ' \rangle &= \ala^2 \|x_i\|^2 + \frac{2\ala \|x_i\|}{Z_i}  \left( e^{\bea} + \sum_{j \neq i}e^{a_{ij}} \langle y_i , y_j \rangle \right)
            \\  & \quad + \frac{1}{Z_i ^2}  \left( e^{2\bea} + 2 e^{\bea} \sum_{j \neq i}e^{a_{ij}} \langle y_i , y_j \rangle + \sum_{j \neq i} \sum_{k \neq i}e^{a_{ij}+a_{ik}} \langle y_k , y_j \rangle \right).
        \end{split}
    \end{align}
Let $\bea = \gamma \log n$ where $\gamma$ is a positive constant. When $\gamma < \frac{1}{1-\rho_1}$,
    \begin{align}\label{e:att length phase 1}
        \langle x_i ' , x_i ' \rangle = \ala^2 \|x_i\|^2 + 2\ala \|x_i\| \frac{\sum_{k \neq i}e^{a_{ik}} \langle y_i , y_k \rangle }{\sum_{k \neq i}e^{a_{ik}}} + \frac{\sum_{k \neq i} \sum_{l \neq i}e^{a_{ik}+a_{il}} \langle y_k , y_l \rangle}{\left( \sum_{k \neq i}e^{a_{ik}} \right)^2} + o_n(1).
    \end{align}
When $\gamma > \frac{1}{1-\rho_2}$, 
    \begin{align}\label{e:att length phase 2}
        \langle x_i ' , x_i ' \rangle = (\ala\|x_i\|+1)^2 + o_n(1).
    \end{align}
In both cases, the terms $o_n(1)$ go to $0$ as $n \to +\infty$ with speeds independent of $i$ but only depending on $\gamma,\rho_1, \rho_2,\ala$.
\end{lemma}
\begin{proof}[Proof of Lemma~\ref{lem:att update length}]
    According to Equation~(\ref{e:att update}), we see that
        \begin{align}
             \langle x_i ' , x_i ' \rangle = \ala^2 \|x_i\|^2 + 2\ala \langle x_i ,\att(y_i) \rangle +  \langle \att(y_i) ,\att(y_i) \rangle.
        \end{align}
    Equation~(\ref{e:att update length}) follows from direct computations. Two phase transitions Equation~(\ref{e:att length phase 1}) and Equation~(\ref{e:att length phase 2}) follow from similar arguments as in Lemma \ref{lem:Z_i asymptotics}.
\end{proof}

\begin{lemma}\label{lem:att update angle}
    Under Assumption \ref{a:equal angle} and Equation~(\ref{e:att update}), for any two different $i, j \in \llbracket 1 , n \rrbracket$,
        \begin{align}\label{e:att update angle}
            \begin{split}
                \langle x_i ' , x_j ' \rangle &=  \ala^2 \langle x_i , x_j \rangle + \frac{\ala \|x_j\| }{ Z_i}  \left( e^{\bea} \langle y_j , y_i \rangle + \sum_{k \neq i} e^{a_{ik}} \langle y_j,y_k \rangle \right) + \frac{\ala \|x_i\| }{ Z_j}  \left( e^{\bea} \langle y_i , y_j \rangle + \sum_{l \neq j} e^{a_{jl}} \langle y_i,y_l \rangle \right)
                \\  &\quad + \frac{1}{Z_i Z_j} \left( e^{2\bea} \langle y_i , y_j \rangle + e^{\bea} \sum_{k \neq i} e^{a_{ik}}\langle y_j , y_k \rangle + e^{\bea} \sum_{l \neq j} e^{a_{jl}}\langle y_i , y_l \rangle + \sum_{k\neq i} \sum_{l \neq j}e^{a_{ik}+a_{jl}} \langle y_k, y_l \rangle \right) .
            \end{split}
        \end{align}
Let $\bea = \gamma \log n$ where $\gamma$ is a positive constant. When $\gamma < \frac{1}{1-\rho_1}$,
    \begin{align}\label{e:att angle phase 1}
        \begin{split}
            \langle x_i ' , x_j ' \rangle &= \ala^2 \langle x_i, x_j \rangle + \ala \|x_j\|\frac{ \sum_{k \neq i} e^{a_{ik}} \langle y_j,y_k \rangle }{ \sum_{k \neq i} e^{a_{ik}}}   + \ala \|x_i\|\frac{ \sum_{l \neq j} e^{a_{jl}} \langle y_i,y_l \rangle }{ \sum_{l \neq j} e^{a_{jl}} } 
            \\  & \quad + \frac{\sum_{k\neq i} \sum_{l \neq j}e^{a_{ik}+a_{jl}} \langle y_k, y_l \rangle}{\left( \sum_{k \neq i} e^{a_{ik}} \right) \cdot  \left( \sum_{l \neq j} e^{a_{jl}} \right)} + o_n(1).
        \end{split}
    \end{align}
When $\gamma > \frac{1}{1-\rho_2}$, 
    \begin{align}\label{e:att angle phase 2}
        \langle x_i ' , x_j ' \rangle = (\ala \|x_i\|+1) (\ala \|x_j\|+1)\langle y_i , y_j \rangle + o_n(1).
    \end{align}
\end{lemma}
\begin{proof}[Proof of Lemma~\ref{lem:att update angle}]
    According to Equation~(\ref{e:att update}), we see that for two different $i, j \in \llbracket 1 , n \rrbracket$,
        \begin{align}
             \langle x_i ' , x_j ' \rangle = \ala^2 p + \ala \langle x_i ,\att(y_j) \rangle + \ala \langle x_j ,\att(y_i) \rangle+ \langle \att(y_i) ,\att(y_j) \rangle.
        \end{align}
    Equation~(\ref{e:att update angle}) follows from direct computations. Two phase transitions Equation~(\ref{e:att angle phase 1}) and Equation~(\ref{e:att angle phase 2}) follow from similar arguments as in Lemma \ref{lem:Z_i asymptotics}.
\end{proof}

Next, we prove Theorem~\ref{thm:att limit phase}.
\begin{proof}[Proof of Theorem~\ref{thm:att limit phase}]
    We first discuss the case when $\gamma < \frac{1}{1-\rho_1}$. According to Equation~(\ref{e:att angle phase 1}) and Assumption \ref{a:equal angle}, we see that
        \begin{align}
            \begin{split}
                \langle x_i ' , x_j ' \rangle &\geq \ala ^2 \|x_i\| \|x_j\| \rho_1 + \ala \|x_j\|\rho_1+ \ala \|x_i\|\rho_1 + \rho_1 +o_n(1)
                \\  &= \rho_1 (\ala \|x_i\|+1) (\ala \|x_j\|+1) + o_n(1).
            \end{split}
        \end{align}
By Equation~(\ref{e:att length phase 1}), we see that
    \begin{align}
        \begin{split}
            \langle x_i ' , x_i ' \rangle &\leq \ala^2 \|x_i\|^2 + 2\ala \|x_i\| \rho_2 + \rho_2 + o_n(1)
            \\  &= \ala^2 \|x_i\|^2 + 2\ala \|x_i\| +1 - (1-\rho_2)(1+2\ala \|x_i\|) + o_n(1)
            \\  &\leq (\ala \|x_i\|+1)^2 - (1-\rho_2)(1+2\ala q_1)+o_n(1).
        \end{split}
    \end{align}
We have a similar inequality for $\langle x_j ' , x_j ' \rangle $. So, there is a constant $\delta>0$ depending on $\rho_2,\ala,q_1,q_2$ and independent of $n$, such that 
    \begin{align}
        \frac{1}{\|x_i'\|} \geq \frac{1+\delta}{\ala \|x_i\|+1} +o_n(1), \text{ and } \frac{1}{\|x_j'\|} \geq \frac{1+\delta}{\ala \|x_j\|+1} +o_n(1).
    \end{align}
Hence,
    \begin{align}
        \langle y_i ' , y_j ' \rangle \geq \rho_1(1+\delta)^2 + o_n(1)\geq \rho_1 + \varepsilon + o_n(1),
    \end{align}
for $\varepsilon = \rho_1(1+2\delta)\delta>0$ independent of $n$.

For the case when $\gamma < \frac{1}{1-\rho_2}$, Equation~(\ref{e:update phase 2 identity}) and Equation~(\ref{e:update phase 2}) follow directly from Lemma~\ref{lem:gamma large single point}, Lemma~\ref{lem:att update length}, and Lemma~\ref{lem:att update angle}.
\end{proof}

\begin{proof}[Proof of Theorem~\ref{thm:att limit phase 2}]
    We notice that Assumption~\ref{a:simplex} corresponds to the special case when $q_1=q_2=q$ and $\rho_1 = \rho_2=\rho$ in Assumption \ref{a:equal angle}. Clearly, $Z_i$ is independent of the choice of $i \in \llbracket 1 , n \rrbracket$ by its definition Equation~(\ref{e:def Zi}). According to the explicit forms Equation~(\ref{e:att update length}) in Lemma \ref{lem:att update length} and Equation~(\ref{e:att update angle}) in Lemma \ref{lem:att update angle}, one directly sees that both $\langle x_i , x_i \rangle$ and $\langle x_i , x_j \rangle$ are independent of the choices of $i,j \in \llbracket 1 , n \rrbracket$. We can further compute that 
    for any $i \in \llbracket 1 , n \rrbracket$,
    \begin{align}\label{e:length update phase}
        \lim_{n \to +\infty} \langle x_i ' , x_i ' \rangle = 
        \begin{cases}
            \ala^2 q+ 2\ala \sqrt{q}\rho + \rho  & \text{if $\gamma < \frac{1}{1-\rho}$},
            \\  \ala^2 q + \ala \sqrt{q} (1+\rho) + \frac{1+3\rho}{4} & \text{if $\gamma = \frac{1}{1-\rho}$},
            \\ (\ala\sqrt{q}+1)^2& \text{if $\gamma > \frac{1}{1-\rho}$},
        \end{cases}
    \end{align}
and for any two different $i,j \in \llbracket 1 , n \rrbracket$, 
    \begin{align}\label{e:angle update phase}
        \lim_{n \to +\infty} \langle x_i ' , x_j ' \rangle = \rho (\ala\sqrt{q}+1)^2 .
    \end{align}
Equation~(\ref{e:simplex_phase}) follows from Equation~(\ref{e:length update phase}) and Equation~(\ref{e:angle update phase}).

When $\gamma < \frac{1}{1-\rho}$, we see that 
    \begin{align}
        \lim_{n \to +\infty} \langle y_i ' , y_j ' \rangle = \frac{\rho (\ala\sqrt{q}+1)^2}{\ala^2 q+ 2\ala \sqrt{q}\rho + \rho} > \frac{\rho (\ala\sqrt{q}+1)^2}{\ala^2 q+ 2\ala \sqrt{q} + 1} = \rho,
    \end{align}
where the strict inequality is because $\rho <1$. When $\gamma = \frac{1}{1-\rho}$, we can similarly show that $\lim_{n \to +\infty} \langle y_i ' , y_j ' \rangle > \rho$. This completes the proof for Theorem \ref{thm:att limit phase 2}.
\end{proof}

\section{Proof of Theorem~\ref{thm:att propagation gradient norm 2} and Theorem~\ref{thm:att propagation gradient norm}}\label{sec:proof of gradient phase transition}

We prove Theorem~\ref{thm:att propagation gradient norm} first. We need to explicitly compute terms in $\frac{\partial (\att(N(x_j)))_v}{\partial (x_i)_u}$, for which we need the following lemmas.

\subsection{Proof of Theorem~\ref{thm:att propagation gradient norm}}
\begin{lemma}\label{lem:norm gradient}
    For any $i,k \in \llbracket 1, n \rrbracket$ and $u,w \in \llbracket 1 , d \rrbracket$,
        \begin{align}
            \frac{\partial (N(x_k))_w}{\partial (x_i)_u} = \delta_{ik} \frac{\delta_{wu}\|x_k\|^2 - (x_k)_w (x_k)_u}{\|x_k\|^3}.
        \end{align}
\end{lemma}
\begin{proof}[Proof of Lemma~\ref{lem:norm gradient}]
    \begin{align}
        \frac{\partial (N(x_k))_w}{\partial (x_i)_u} = \frac{\partial ((x_k)_w \cdot \|x_k\|^{-1})}{\partial (x_i)_u} = \delta_{ik} \frac{\delta_{wu}\|x_k\| - (x_k)_w \cdot \frac{(x_k)_u}{\|x_k\|}}{\|x_k\|^2 }.
    \end{align}
\end{proof}

\begin{lemma}\label{lem:att gradient}
    For any $k,j \in \llbracket 1, n \rrbracket$ and $w,v \in \llbracket 1 , d \rrbracket$,
        \begin{align}
            \begin{split}
                &\frac{\partial (\att(y_j))_v}{\partial (y_k)_w} 
                \\&= \bigg[\left(\delta_{kj}\bea\left( \sum_{m=1} ^n e^{\bea \langle y_j , y_m\rangle} (y_m)_w (y_m)_v \right) + e^{\bea \langle y_j , y_k\rangle} (\bea (y_j)_w (y_k)_v+\delta_{wv})\right)\cdot \left( \sum_{l=1} ^n e^{\bea \langle y_j,y_l \rangle} \right)
                \\  & \quad - \left(\delta_{kj} \bea \left(\sum_{l=1} ^n  e^{\bea \langle y_j,y_l \rangle} (y_l)_w \right)+ \bea  e^{\bea \langle y_j,y_k \rangle}(y_j)_w\right)  \cdot \left( \sum_{m=1} ^n e^{\bea \langle y_j,y_m \rangle} (y_m)_v \right) \bigg] \\ & \quad \cdot \left( \sum_{l=1} ^n e^{\bea \langle y_j,y_l \rangle} \right)^{-2}.
            \end{split}
        \end{align}
\end{lemma}
\begin{proof}[Proof of Lemma~\ref{lem:att gradient}]
    By Equation~(\ref{e:att}), 
        \begin{align}
            (\att(y_j))_v = \frac{\sum_{m=1} ^n e^{\bea \langle y_j,y_m \rangle} (y_m)_v}{\sum_{l=1} ^n e^{\bea \langle y_j,y_l \rangle}}.
        \end{align}
    A direct computation shows that
        \begin{align}
            \begin{split}
                &\frac{\partial (\att(y_j))_v}{\partial (y_k)_w}
                \\ &= \bigg[\left(\sum_{m=1} ^n \left(\delta_{kj}\bea (y_m)_w (y_m)_v + \delta_{km} \bea (y_j)_w (y_m)_v+ \delta_{km}\delta_{wv}\right)e^{\bea \langle y_j , y_m\rangle}\right)\cdot \left( \sum_{l=1} ^n e^{\bea \langle y_j,y_l \rangle} \right)
                \\  & \quad - \left(\sum_{l=1} ^n  \left(\delta_{kj}(y_l)_w + \delta_{kl} (y_j)_w\right)\bea e^{\bea \langle y_j,y_l \rangle} \right) \cdot \left( \sum_{m=1} ^n e^{\bea \langle y_j,y_m \rangle} (y_m)_v \right) \bigg] \\ & \quad \cdot \left( \sum_{l=1} ^n e^{\bea \langle y_j,y_l \rangle} \right)^{-2}
                \\ &= \bigg[\left(\delta_{kj}\bea \sum_{m=1} ^n e^{\bea \langle y_j , y_m\rangle} (y_m)_w (y_m)_v + e^{\bea \langle y_j , y_k\rangle} (\bea (y_j)_w (y_k)_v+\delta_{wv})\right)\cdot \left( \sum_{l=1} ^n e^{\bea \langle y_j,y_l \rangle} \right)
                \\  & \quad - \left(\delta_{kj} \bea \sum_{l=1} ^n  e^{\bea \langle y_j,y_l \rangle} (y_l)_w + \bea  e^{\bea \langle y_j,y_k \rangle}(y_j)_w\right)  \cdot \left( \sum_{m=1} ^n e^{\bea \langle y_j,y_m \rangle} (y_m)_v \right) \bigg] \\ & \quad \cdot \left( \sum_{l=1} ^n e^{\bea \langle y_j,y_l \rangle} \right)^{-2}.
            \end{split}
        \end{align}
\end{proof}
 For $x,y \in \R^d$, we use $x \otimes y$ to denote the $d \times d$ matrix with $(u,v)$-th element $(x \otimes y)_{uv} = (x)_u (y)_v$, i.e., $x \otimes y \coloneqq x y^T$. We then have the following proposition.
\begin{lemma}\label{lem:att norm chain}
    Adopt Assumption \ref{a:equal angle} and Equation~(\ref{e:att update}). 
    For any $i,j \in \llbracket 1, n \rrbracket$, consider the $d \times d$ matrix formed by  $\frac{\partial (\att(N(x_j)))_v}{\partial (x_i)_u}$, for $u,v \in \llbracket 1 , d \rrbracket$. Denote $y_k = N(x_k)$ for each $k \in \llbracket 1, n \rrbracket$. Then, this matrix has the following form:
        \begin{align}\label{e:att jacobian}
             \left( \frac{\partial (\att(N(x_j)))_v}{\partial (x_i)_u} \right)_{d \times d} = \|x_i\|^{-\frac{1}{2}}\left[(\bfr_1 +\bfr_2)Z_j - (\bfu_1 + \bfu_2)\otimes \sumv_j \right] \cdot Z_j^{-2},
        \end{align}
    where $Z_j = \sum_{l=1} ^n e^{\bea \langle y_j , y_l \rangle}$ as in Equation~(\ref{e:def Zi}),
    \begin{align}\label{e:notation R12}
        \bfr_1 \coloneqq \delta_{ij}\bea  \left( \sumt_j - y_i \otimes (\sumt_j y_i)  \right)  , \quad \bfr_2 \coloneqq e^{\bea \langle y_j , y_i\rangle} \left((-y_i+\bea \proj_{y_i} y_j) \otimes y_i+I_d  \right),
    \end{align}
and 
    \begin{align}\label{e:notation U1234}
        \begin{split}
            \bfu_1 \coloneqq \delta_{ij} \bea     \left(\proj_{y_i} \sumv_j \right), \quad \bfu_2 \coloneqq \bea  e^{\bea \langle y_j,y_i \rangle}\left(\proj_{y_i} y_j \right).
        \end{split}
    \end{align}
In Equation~(\ref{e:notation R12}) and Equation~(\ref{e:notation U1234}),
    \begin{align}\label{e:simple V,W}
        \sumv_j \coloneqq \sum_{m=1} ^n e^{\bea \langle y_j , y_m \rangle} y_m, \quad \sumt_j \coloneqq \sum_{m=1} ^n e^{\bea \langle y_j , y_m \rangle} y_m \otimes y_m, \quad \proj_x y \coloneqq y - \langle y , x \rangle x.
    \end{align}
\end{lemma}
\begin{proof}[Proof of Lemma~\ref{lem:att norm chain}]
    By chain rule and Proposition \ref{lem:norm gradient}, we have that
    \begin{align}
        \begin{split}
            \frac{\partial (\att(N(x_j)))_v}{\partial (x_i)_u} &= \sum_{k=1} ^n \sum_{w = 1} ^d \frac{\partial (\att(y_j))_v}{\partial (y_k)_w} \bigg|_{Y=\cln(X)} \cdot \frac{\partial (N(x_k))_w}{\partial (x_i)_u}
            \\  &= \|x_i\|^{-\frac{3}{2}} \left(\|x_i\| \cdot \frac{\partial (\att(y_j))_v}{\partial (y_i)_u}  - \sum_{w=1} ^d (x_i)_u (x_i)_w \frac{\partial (\att(y_j))_v}{\partial (y_i)_w} \right) \bigg|_{Y=\cln(X)}
            \\  &= \|x_i\|^{-\frac{1}{2}} \left( \frac{\partial (\att(y_j))_v}{\partial (y_i)_u}  - (y_i)_u \sum_{w=1} ^d  (y_i)_w \frac{\partial (\att(y_j))_v}{\partial (y_i)_w} \right) \bigg|_{Y=\cln(X)}.
        \end{split}
    \end{align}
According to Proposition \ref{lem:att gradient} and the notation $Z_j = \sum_{l=1} ^n e^{a_{jl}}$, we see that
            \begin{align}
            \begin{split}
                &\sum_{w=1} ^d (y_i)_w\frac{\partial (\att(y_j))_v}{\partial (y_i)_w} 
                \\&= \bigg[\left(\delta_{ij}\bea \left( \sum_{m=1} ^n e^{\bea \langle y_j , y_m\rangle} \langle y_m , y_i \rangle (y_m)_v \right) + e^{\bea \langle y_j , y_i\rangle} (\bea \langle y_j, y_i \rangle +1)(y_i)_v\right)\cdot Z_j
                \\  & \quad - \left(\delta_{ij} \bea \left(\sum_{l=1} ^n  e^{\bea \langle y_j,y_l \rangle} \langle y_l , y_i \rangle \right) + \bea  e^{\bea \langle y_j, y_i \rangle} \langle y_j , y_i \rangle \right)  \cdot \left( \sum_{m=1} ^n e^{\bea \langle y_j,y_m \rangle} (y_m)_v \right) \bigg]  \cdot Z_j ^{-2}.
            \end{split}
        \end{align}
Hence,
    \begin{align}\label{e:grad ij}
        \begin{split}
        &\|x_i\|^{\frac{1}{2}} \cdot \frac{\partial (\att(N(x_j)))_v}{\partial (x_i)_u}
        \\  &=\left( \frac{\partial (\att(y_j))_v}{\partial (y_i)_u}  - (y_i)_u \sum_{w=1} ^d  (y_i)_w \frac{\partial (\att(y_j))_v}{\partial (y_i)_w} \right) \bigg|_{Y=\cln(X)}
            \\  &= \bigg[\bigg[\delta_{ij}\bea\left( \sum_{m=1} ^n e^{\bea \langle y_j , y_m\rangle} \left( (y_m)_u (y_m)_v - \langle y_m , y_i \rangle (y_m)_v (y_i)_u \right) \right) 
            \\  & \quad + e^{\bea \langle y_j , y_i\rangle} (\bea (y_j)_u (y_i)_v+\delta_{uv} - \left(\bea \langle y_j, y_i \rangle +1 \right)(y_i)_v (y_i)_u )\bigg]\cdot Z_j
                \\  & \quad - \bigg[\delta_{ij} \bea \left(\sum_{l=1} ^n  e^{\bea \langle y_j,y_l \rangle} \left((y_l)_u - \langle y_l , y_i \rangle (y_i)_u\right) \right)
                \\  & \quad + \bea  e^{\bea \langle y_j,y_i \rangle}\left((y_j)_u-\langle y_j , y_i \rangle (y_i)_u \right) \bigg]  \cdot \left( \sum_{m=1} ^n e^{\bea \langle y_j,y_m \rangle} (y_m)_v \right) \bigg]  \cdot Z_j ^{-2}.
        \end{split}
    \end{align}
We then adopt the notation Equation~(\ref{e:simple V,W}), i.e.,
    \begin{align}
        \sumv_j = \sum_{m=1} ^n e^{\bea \langle y_j , y_m \rangle} y_m, \quad \sumt_j = \sum_{m=1} ^n e^{\bea \langle y_j , y_m \rangle} y_m \otimes y_m, \quad \proj_x y \coloneqq y - \langle y , x \rangle x.
    \end{align}
So, the matrix form of Equation~(\ref{e:grad ij}) becomes
    \begin{align}
        \begin{split}
            &\bigg[\bigg[\delta_{ij}\bea   \left( \sumt_j - y_i \otimes (\sumt_j y_i)  \right)  
            + e^{\bea \langle y_j , y_i\rangle} (\bea y_j \otimes y_i+I_d - \left(\bea \langle y_j, y_i \rangle +1 \right) y_i \otimes y_i )\bigg]\cdot Z_j
            \\  & - \bigg[\delta_{ij} \bea   \left(\sumv_j - \langle \sumv_j , y_i \rangle y_i\right) 
             + \bea  e^{\bea \langle y_j,y_i \rangle}\left(y_j-\langle y_j , y_i \rangle y_i \right) \bigg]  \otimes \sumv_j \bigg]  \cdot Z_j ^{-2}
             \\ &= \bigg[\bigg[\delta_{ij}\bea   \left( \sumt_j - y_i \otimes (\sumt_j y_i)  \right)  
            + e^{\bea \langle y_j , y_i\rangle} ((-y_i+\bea \proj_{y_i} y_j) \otimes y_i+I_d  )\bigg]\cdot Z_j
            \\  & - \bigg[\delta_{ij} \bea     \left(\proj_{y_i} \sumv_j \right) 
             + \bea  e^{\bea \langle y_j,y_i \rangle} \left(\proj_{y_i} y_j \right) \bigg]  \otimes \sumv_j \bigg]  \cdot Z_j ^{-2}.
        \end{split}
    \end{align}
We further use the notations in Equation~(\ref{e:notation R12}) and Equation~(\ref{e:notation U1234}), i.e., 
    \begin{align}
        \bfr_1 = \delta_{ij}\bea e^{\bea \rho}  \left( \sumt_j - y_i \otimes (\sumt_j y_i)  \right)  , \quad \bfr_2 = e^{\bea \langle y_j , y_i\rangle} \left((-y_i+\bea \proj_{y_i} y_j) \otimes y_i+I_d  \right),
    \end{align}
and 
    \begin{align}
        \begin{split}
            &\bfu_1 = \delta_{ij} \bea     \left(\proj_{y_i} \sumv_j \right), \quad \bfu_2 = \bea  e^{\bea \langle y_j,y_i \rangle}\left(\proj_{y_i} y_j \right).
        \end{split}
    \end{align}
Finally, the matrix form of Equation~(\ref{e:grad ij}) becomes 
    \begin{align}\label{e:grad ij 2}
        \left[(\bfr_1 +\bfr_2)Z - (\bfu_1 + \bfu_2)\otimes \sumv_j \right] \cdot Z_j^{-2}.
    \end{align}

\end{proof}

\begin{lemma}\label{lem:gamma large single point gradient}
    Let $\bea = \gamma \log n$ where $\gamma$ is a positive constant. Under Assumption \ref{a:equal angle} and Equation~(\ref{e:att update}), if $\gamma > \frac{1}{1-\rho_2}$, then for any fixed $i,j \in \llbracket 1 , n \rrbracket$, the $d \times d$ matrix satisfies
        \begin{align}
            \left(\frac{\partial (\att(N(x_j)))_v}{\partial (x_i)_u} \right)_{d \times d} =  \frac{\delta_{ij}}{\|x_i\|} \left( I_d - y_i \otimes y_i \right) + \bfo_n(1) + o_n(1) \cdot  I_d, 
        \end{align}
    where the leading order term is exactly $ \frac{\partial (N(x_j))_v}{\partial (x_i)_u}$. The term $\bfo_n(1)$ ($o_n(1)$, respectively) is a $d \times d $ matrix (constant, respectively) with matrix norm as defined in Equation~(\ref{e:def jacobian norm}) (value, respectively) going to $0$ as $n \to +\infty$, with a speed independent of $i,j$ but only depending on $\gamma,\rho_2,q_1$.
\end{lemma}
\begin{proof}[Proof of Lemma~\ref{lem:gamma large single point gradient}]
    We frequently use this formula: for two vectors $V_1,V_2$, the matrix norm of $V_1 \otimes V_2$ as defined in Equation~(\ref{e:def jacobian norm}) is $\|V_1\|\|V_2\|$. 
    When $\gamma > \frac{1}{1-\rho_2}$, $n^{\gamma} > n^{1+\gamma\rho_2}$, and we know from Lemma \ref{lem:Z_i asymptotics} that $Z_j = (1+o_n(1)) \cdot e^{\bea}$ for any $j \in \llbracket 1 , n \rrbracket$. Adopt the notations in Proposition \ref{lem:att norm chain}, we then show the following facts when $\gamma > \frac{1}{1-\rho_2}$:
        \begin{align}
            \bfr_1 Z_j ^{-1} = \bfo_n(1), \quad \bfr_2 Z_j ^{-1} = \delta_{ij} \left( -y_i \otimes y_i+I_d \right)+ \bfo_n(1) + o_n(1) \cdot I_d ,
        \end{align}
    and 
        \begin{align}
            \left[ (\bfu_1 + \bfu_2)\otimes \sumv_j \right] \cdot Z_j ^{-2} = \bfo_n(1).
        \end{align}
        
First, for $\bfr_1 Z_j ^{-1}$, when $i \neq j$, we have that $\bfr_1 = 0$ by its definition. When $i=j$, $\bfr_1 = \bea \sum_{m=1} ^n e^{\bea \langle y_i , y_m\rangle} \left( y_m \otimes y_m - \langle y_m , y_i \rangle y_i\otimes y_m\right)$ and we notice that the term when $m=i$ is $0$. So, because $\| y_m \otimes y_m - \langle y_m , y_i \rangle y_i\otimes y_m \| \leq \|y_m\|^2 + \|y_m\|^2\|y_i\|^2 = 2$, $e^{\bea} = n^{\gamma}$,
    \begin{align}
        \|\bfr_1\| Z_j ^{-1} \leq \bea (n-1)e^{\bea \rho_2}\cdot 2  Z_j^{-1} \leq 2 \gamma \log(n) \cdot n^{\gamma \rho_2+1 - \gamma}(1+o_n(1)),
    \end{align}
which goes to $0$ with a speed independent of $i,j$, because $\gamma \rho_2+1 - \gamma <0$.

For $\bfr_2 Z_j ^{-1}$, we notice that when $i \neq j$, $ e^{\bea \langle y_i , y_j \rangle} Z_j ^{-1} \leq e^{\bea (\rho_2-1)}(1+o_n(1)) = n^{\gamma (\rho_2 - 1)}(1+o_n(1))$, which goes to $0$ with a speed independent of $i,j$. So, $\bfr_2 Z_j ^{-1} =  \bfo_n(1) + o_n(1) \cdot  I_d$ when $i \neq j$. When $i=j$, $\bfr_2 = e^{\bea} \left( -y_i \otimes y_i+I_d \right)$, and so $\bfr_2 Z_j ^{-1} = (-y_i \otimes y_i+I_d) + \bfo_n(1) + o_n(1) \cdot  I_d$.

For $\left[ (\bfu_1 + \bfu_2)\otimes \sumv_j \right] \cdot Z_j ^{-2}$, we see that when $i \neq j$, $\bfu_1 =0$, and so
    \begin{align}
        \begin{split}
            &\left\| (\bfu_1 + \bfu_2)\otimes \sumv_j \right\| \cdot Z_j ^{-2} \leq Z_j ^{-2} \bea \sum_{m=1} ^n e^{\bea \langle y_j , y_m + y_i\rangle} \| y_m \| \| \proj_{y_i} y_j\| 
            \\  &\leq Z_j ^{-2} \bea e^{\bea (1+\rho_2)}n = \gamma \log(n) n^{\gamma(\rho_2-1) + 1}(1+o_n(1)),
        \end{split}
    \end{align}
which goes to $0$ with a speed independent of $i,j$ because $\gamma > \frac{1}{1-\rho_2}$. When $i=j$, $\bfu_2 = 0$, and so
    \begin{align}
        \begin{split}
            &\left\| (\bfu_1 + \bfu_2)\otimes \sumv_j \right\| \cdot Z_j ^{-2} \leq Z_j ^{-2} \bea \|\proj_{y_i} \sumv_i\| \|\sumv_i\| 
            \\  &\leq  Z_j ^{-2} \bea \left(\sum_{m \neq i} e^{\bea \langle y_i , y_m \rangle}  \|\proj_{y_i} y_m\| \right) \cdot \left(e^{\bea} + \sum_{m \neq i} e^{\bea \langle y_i , y_m \rangle}  \|\proj_{y_i} y_m\| \right)
            \\  &\leq Z_j ^{-2} \bea \left(e^{\bea \rho_2}n \right) \cdot \left(e^{\bea}+ e^{\bea \rho_2}n \right) 
            \\  &= \gamma \log(n) n^{\gamma(\rho_2-1) + 1}(1+n^{\gamma(\rho_2-1) + 1})(1+o_n(1)),
        \end{split}
    \end{align}
which goes to $0$ with a speed independent of $i,j$ because $\gamma > \frac{1}{1-\rho_2}$. Hence, $\left[ (\bfu_1 + \bfu_2)\otimes \sumv_j \right] \cdot Z_j ^{-2} = \bfo_n(1)$.
\end{proof}

\begin{lemma}\label{lem:gamma small all point gradient}
    Let $\bea = \gamma \log n$ where $\gamma$ is a positive constant. Under Assumption \ref{a:equal angle} and Equation~(\ref{e:att update}), if $\gamma < \frac{1}{1-\rho_1}$, then for  fixed $i,j \in \llbracket 1 , n \rrbracket$, the $d \times d$ matrix satisfies
        \begin{align}
            \left\| \left(\frac{\partial (\att(N(x_j)))_v}{\partial (x_i)_u} \right)_{d \times d} \right\| \leq \|x_i\|^{-\frac{1}{2}} \cdot \left( 2 \bea \delta_{ij}+ (2\bea + \sqrt{d}) e^{a_{ij}}Z_j ^{-1} \right).
        \end{align}
\end{lemma}
\begin{proof}[Proof of Lemma~\ref{lem:gamma small all point gradient}]
    According to Lemma \ref{lem:Z_i asymptotics}, when $\gamma < \frac{1}{1-\rho_1}$, $Z_j = (1+o_n(1)) \cdot \left(\sum_{k \neq j} e^{a_{jk}}\right)$ for any $j \in \llbracket 1 , n \rrbracket$, and $Z_j \geq n^{\gamma \rho_1 + 1} (1+o_n(1)) > n^\gamma (1+o_n(1))$, because $\gamma \rho_1 + 1 > \gamma$. Adopt the notations in Proposition \ref{lem:att norm chain}, we then show the following facts when $\gamma < \frac{1}{1-\rho_1}$:
        \begin{align}
            \|\bfr_1\| Z_j ^{-1} \leq \delta_{ij} \bea, \quad \|\bfr_2\| Z_j ^{-1} \leq Z_j ^{-1} e^{ a_{ij}} \left(\bea+\sqrt{d-1}  \right),
        \end{align}
    and 
        \begin{align}
            \left\| (\bfu_1 + \bfu_2)\otimes \sumv_j \right\| \cdot Z_j ^{-2} \leq \bea \left(\delta_{ij} + e^{a_{ij}} Z_j ^{-1}\right) .
        \end{align}

    First, for $\bfr_1 Z_j ^{-1}$, when $i \neq j$, we have that $\bfr_1 = 0$ by its definition. When $i=j$, $\bfr_1 = \bea \sum_{m=1} ^n e^{\bea \langle y_i , y_m\rangle} \left( y_m \otimes y_m - \langle y_m , y_i \rangle y_i\otimes y_m\right)$. So, because we have that $\| y_m \otimes y_m - \langle y_m , y_i \rangle y_i\otimes y_m \| = \|\proj_{y_i} y_n \otimes y_m\|= \|\proj_{y_i} y_n\| \|y_m\| \leq 1$,
    \begin{align}
        \|\bfr_1\| Z_j ^{-1} \leq \bea Z_j\cdot  \cdot   Z_j^{-1} = \bea.
    \end{align}

For $\bfr_2 Z_j ^{-1}$, because $\|-y_i \otimes y_i +I_d\| = \sqrt{d-1}$, we have that 
    \begin{align}
        \begin{split}
            \|\bfr_2\| Z_j ^{-1} &\leq Z_j ^{-1} e^{ a_{ij}} \left(\bea+\sqrt{d-1}  \right) .
        \end{split}
    \end{align}

For $\left[ (\bfu_1 + \bfu_2)\otimes \sumv_j \right] \cdot Z_j ^{-2}$, we see that $\|\sumv_j\| \leq \sum_{m=1} ^n e^{\bea \langle y_j , y_m \rangle} = Z_j$. Also, $\|\bfu_1\| Z_j ^{-1} \leq \delta_{ij}\bea \|\sumv_j\| Z_j ^{-1} \leq \delta_{ij}\bea$, $\|\bfu_2\| Z_j ^{-1} \leq \bea e^{a_{ij}} Z_j ^{-1} $. Hence, we have that
    \begin{align}
        \left\| (\bfu_1 + \bfu_2)\otimes \sumv_j \right\| \cdot Z_j ^{-2} \leq \bea \left(\delta_{ij} + e^{a_{ij}} Z_j ^{-1}\right).
    \end{align}
\end{proof}

\begin{proof}[Proof of Theorem~\ref{thm:att propagation gradient norm}]
    Theorem~\ref{thm:att propagation gradient norm} follows directly from Lemma~\ref{lem:gamma large single point gradient} and Lemma~\ref{lem:gamma small all point gradient}.
\end{proof}

\subsection{Proof of Theorem~\ref{thm:att propagation gradient norm 2}}\label{sec:proof gradient norm 2}

The proof for Theorem~\ref{thm:att propagation gradient norm 2} requires more delicate arguments. The part when $\gamma > \frac{1}{1-\rho}$ in Theorem~\ref{thm:att propagation gradient norm 2} directly follows from Lemma~\ref{lem:gamma large single point gradient}, so we only focus on the part when $\gamma \leq \frac{1}{1-\rho}$. We remark that when $\gamma < \frac{1}{1-\rho}$, our result is that $\frac{1}{nd}\|\nabla_X X'\|^2=  0+o_n(1)$, which is a better estimate than Equation~(\ref{e:propagation gradient phase 1}) in Theorem~\ref{thm:att propagation gradient norm}.

We first have the following lemma which replaces Lemma~\ref{lem:att norm chain} when we adopt Assumption~\ref{a:simplex}.

\begin{lemma}\label{lem:att norm chain 2}
    Adopt Assumption \ref{a:simplex} and Equation~(\ref{e:att update}). 
    For any $i,j \in \llbracket 1, n \rrbracket$, consider the $d \times d$ matrix formed by  $\frac{\partial (\att(N(x_j)))_v}{\partial (x_i)_u}$, for $u,v \in \llbracket 1 , d \rrbracket$. Denote $y_k = N(x_k)$ for each $k \in \llbracket 1, n \rrbracket$. Then, this matrix has the following form:
        \begin{align}\label{e:att jacobian simplex}
             \left( \frac{\partial (\att(N(x_j)))_v}{\partial (x_i)_u} \right)_{d \times d} = q^{-\frac{1}{2}}\left[(\bfr_1 +\bfr_2)Z - (\bfu_1 + \bfu_2)\otimes (\bfu_3 + \bfu_4) \right] \cdot Z^{-2},
        \end{align}
    where $Z = e^{\bea} + (n-1)e^{\bea \rho}$,
    \begin{align}\label{e:notation R12 simplex}
        \bfr_1 \coloneqq \delta_{ij}\bea e^{\bea \rho} \left( \sumt - y_i \otimes (\sumt y_i)  \right)  , \quad \bfr_2 \coloneqq e^{\bea \langle y_j , y_i\rangle} \left((-y_i+\bea \proj_{y_i} y_j) \otimes y_i+I_d  \right),
    \end{align}
and 
    \begin{align}\label{e:notation U1234 simplex}
        \begin{split}
            &\bfu_1 \coloneqq \delta_{ij} \bea e^{\bea \rho}    \left(\proj_{y_i} \sumv \right), \quad \bfu_2 \coloneqq \bea  e^{\bea \langle y_j,y_i \rangle}\left(\proj_{y_i} y_j \right),
            \\  & \bfu_3 \coloneqq (e^{\bea}-e^{\bea \rho})y_j, \quad \bfu_4 \coloneqq e^{\bea \rho} \sumv.
        \end{split}
    \end{align}
In Equation~(\ref{e:notation R12 simplex}) and Equation~(\ref{e:notation U1234 simplex}),
    \begin{align}\label{e:simple V,W simplex}
        \sumv \coloneqq \sum_{m=1} ^n y_m, \quad \sumt \coloneqq \sum_{m=1} ^n y_m \otimes y_m, \quad \proj_x y \coloneqq y - \langle y , x \rangle x.
    \end{align}
\end{lemma}
\begin{proof}[Proof of Lemma~\ref{lem:att norm chain 2}]
    We first apply Lemma~\ref{lem:att norm chain} to get Equation~(\ref{e:att jacobian}). After replacing $\langle y_j ,y_m \rangle = \rho$ for $m \neq j$, we can obtain Equation~(\ref{e:att jacobian simplex}). The only remark is that the term $\delta_{ij}(\sumt_j - y_i \otimes (\sumt_j y_i))$ in $\bfr_1$ of Equation~(\ref{e:notation R12}) is nonzero when $i=j$. Then, when $i=j$, $\sumt_i - y_i \otimes (\sumt_i y_i) = \sum_{m=1} ^n e^{\bea \langle y_i , y_m \rangle} (y_m \otimes y_m-y_i \otimes y_m \langle y_m , y_i\rangle)$. If $m=i$, the summand $(y_m \otimes y_m-y_i \otimes y_m \langle y_m , y_i\rangle)$ becomes $0$. Hence, $\sumt_i - y_i \otimes (\sumt_i y_i) = \sum_{m\neq i} e^{\bea \langle y_i , y_m \rangle} (y_m \otimes y_m-y_i \otimes y_m \langle y_m , y_i\rangle) = e^{\bea \rho} \sum_{m\neq i} (y_m \otimes y_m-y_i \otimes y_m \langle y_m , y_i\rangle) = e^{\bea \rho} \left( \sumt - y_i \otimes (\sumt y_i)  \right)$.
\end{proof}

Next, to compute the matrix norm of Equation~(\ref{e:att jacobian simplex}), we see that for any matrix $K$, its matrix norm square equals to $\tr(K^T K)$. Hence, the matrix norm square of Equation~(\ref{e:att jacobian simplex}) equals to
    \begin{align}\label{e:jacobian norm square expansion simplex}
        \begin{split}
            & q^{-1}Z^{-4}\cdot \bigg(\tr\left[ Z^2  (\bfr_1+\bfr_2)^T (\bfr_1+\bfr_2)\right] - 2 Z (\bfu_1 + \bfu_2)^T (\bfr_1+\bfr_2) (\bfu_3 + \bfu_4) 
            \\ & + \|\bfu_1 + \bfu_2\|^2\|\bfu_3 + \bfu_4\|^2  \bigg).
        \end{split}
    \end{align}
We then compute these terms separately, and sum them in $i,j$. We first have the following basic equalities for the notations $\sumv,\sumt$ in Equation~(\ref{e:simple V,W simplex}).
\begin{lemma}\label{lem:basic equalities for VW simplex}
    For the notations in Equation~(\ref{e:simple V,W simplex}), i.e.,
        \begin{align}
        \sumv \coloneqq \sum_{m=1} ^n y_m, \quad \sumt \coloneqq \sum_{m=1} ^n y_m \otimes y_m, \quad \proj_x y \coloneqq y - \langle y , x \rangle x,
    \end{align}
    we have that 
        \begin{align}
            \begin{split}
                &\tr (\sumt^2) = \sum_{m,l} \langle y_m , y_l \rangle^2 = n(n\rho^2 + (1-\rho^2)),
                \\  & \tr(\sumt) = n,
                \quad \tr(\sumt y_i y_i^T) = n \rho^2 + (1-\rho^2), \quad \|\proj_{y_i} y_j \|^2 = 1-\rho^2.
            \end{split}
        \end{align}
    Also,
        \begin{align}
            \begin{split}
                &\sumt y_i = \sum_{m=1} ^n \langle y_m ,y_i \rangle y_m = (1-\rho)y_i +\rho \sumv,
                \\  &\langle \sumv , y_i \rangle =  \sum_{m=1} ^n \langle y_m ,y_i \rangle = n \rho + (1-\rho),
                \\  & \|\sumv\|^2 = \sum_{m,l} \langle y_m , y_l \rangle = n + \rho n(n-1) = n (n\rho + (1-\rho)),
                \\  & \|\proj_{y_i} \sumv \|^2 = \|\sumv\|^2 - \langle \sumv , y_i \rangle^2 = (n-1)(n\rho + (1-\rho))(1-\rho),
                \\  &\|\sumt y_i\|^2 = n^2 \rho^3 +3n\rho^2(1-\rho)+(1+2\rho)(1-\rho)^2 .
            \end{split}
        \end{align}
\end{lemma}
\begin{proof}[Proof of Lemma~\ref{lem:basic equalities for VW simplex}]
    Direct Computations.
\end{proof} 


\begin{lemma}\label{lem:att grad norm 1}
    For terms $\bfr_1,\bfr_2$ in Lemma~\ref{lem:att norm chain 2}, we have that
        \begin{align}\label{e:att grad norm 1}
            \begin{split}
                &\sum_{i,j} \tr\left[  (\bfr_1+\bfr_2)^T (\bfr_1+\bfr_2)\right] 
                \\  &= \bea^2 e^{2\bea \rho}  n \left[n^2\rho^2(1-\rho) + n(1-\rho)(1+\rho-3\rho^2)- (1+2\rho)(1-\rho)^2 \right]
                \\  &+ \bea e^{\bea (\rho+1)} n(n-1)(1-\rho^2)
                \\  &+ e^{2 \bea}(d-1)n+ e^{2\bea \rho} \left[ \bea^2 (1-\rho^2) +d-1\right]n(n-1).
            \end{split}
        \end{align}
As a corollary, when we pick $\bea = \gamma \log n$, we have the following phase transition limits as $n \to +\infty$:
    \begin{align}\label{e:att grad norm phase 1}
         \frac{1}{nZ^2} \sum_{i,j} \tr\left[  (\bfr_1+\bfr_2)^T (\bfr_1+\bfr_2)\right] = 
            \begin{cases}
                \bea^2 \rho^2(1-\rho) +o_n(1)  & \text{if $\gamma < \frac{1}{1-\rho}$},
                \\ \frac{d-1 + \bea^2 \rho^2(1-\rho)}{4} +o_n(1)& \text{if $\gamma = \frac{1}{1-\rho}$},
            \\ d-1 +o_n(1)& \text{if $\gamma > \frac{1}{1-\rho}$}.
            \end{cases}
    \end{align}
\end{lemma}
\begin{proof}[Proof of Lemma~\ref{lem:att grad norm 1}]
    We first notice that $\sumt$ is a symmetric matrix and $\|y_i\|=1$.
    We then expand each term in Lemma~\ref{lem:att grad norm 1} and use Lemma~\ref{lem:basic equalities for VW simplex}.
        \begin{align}
            \begin{split}
                &\sum_{i,j}\tr\left[  (\bfr_1)^T \bfr_1\right] = \bea^2 e^{2\bea \rho}   \sum_i \left( \tr \left( \sumt^2 - 2 \sumt y_i  (\sumt y_i)^T  \right) + \|y_i\|^2 \|\sumt y_i\|^2 \right)
                \\  &= \bea^2 e^{2\bea \rho}   \sum_i \left( \tr \sumt^2 -  2 \|\sumt y_i\|^2+  \|\sumt y_i\|^2 \right)
                \\  &= \bea^2 e^{2\bea \rho}  n \left[n^2\rho^2(1-\rho) + n(1-\rho)(1+\rho-3\rho^2)- (1+2\rho)(1-\rho)^2 \right].
            \end{split}
        \end{align}
Then, 
    \begin{align}
        \begin{split}
            &\sum_{i,j}\tr\left[  (\bfr_1)^T \bfr_2\right] = \tr \sum_{i,j}\delta_{ij}\bea e^{\bea \rho}  \left( \sumt - \sumt y_i y_i ^T  \right) e^{\bea \langle y_j , y_i\rangle} \left((-y_i+\bea \proj_{y_i} y_j) \otimes y_i+I_d  \right) 
            \\  &= \bea e^{\bea (\rho+1)} \tr\sum_i \left( \sumt - \sumt y_i y_i ^T  \right)(-y_i y_i^T + I_d) = \bea e^{\bea (\rho+1)} \tr\sum_i \left( \sumt - \sumt y_i y_i ^T  \right)
            \\  &= \bea e^{\bea (\rho+1)} n(n-1)(1-\rho^2),
        \end{split}
    \end{align}
where the second equality is because $\proj_{y_i} y_i = 0$.
    \begin{align}
        \begin{split}
            &\sum_{i,j}\tr\left[  (\bfr_2)^T \bfr_2\right] = \sum_{i,j} e^{2\bea \langle y_j , y_i\rangle} \tr \left[(-y_i+\bea \proj_{y_i} y_j) y_i ^T +I_d  \right) \left(y_i (-y_i+\bea \proj_{y_i} y_j) ^T +I_d  \right]
            \\  &= \sum_{i \neq j} e^{2\bea \rho} \left[\left(1+\bea^2 (1-\rho^2) \right) -2 + d\right] + \sum_i e^{2 \bea}(d-1) 
            \\  &= e^{2 \bea}(d-1)n+ e^{2\bea \rho} \left[ \bea^2 (1-\rho^2) +d-1\right]n(n-1).
        \end{split}
    \end{align}

Next, we show the asymptotics Equation~(\ref{e:att grad norm phase 1}) as $n \to +\infty$. According to Lemma~\ref{lem:Z_i asymptotics}, we have that 
    \begin{align}\label{e:Z phase}
        Z = 
        \begin{cases}
           (1+o_n(1)) \cdot ne^{\bea \rho}   & \text{if $\gamma < \frac{1}{1-\rho}$},
            \\ (1+o_n(1)) \cdot e^{\bea}  & \text{if $\gamma > \frac{1}{1-\rho}$}.
        \end{cases}
    \end{align}
That is, when $\gamma < \frac{1}{1-\rho}$, the leading order terms are those terms involving $ne^{\bea \rho}$, and all the remaining terms go to $0$ after dividing $ne^{\bea \rho}$; when $\gamma > \frac{1}{1-\rho}$, the leading order terms are those terms involving $e^{\bea}$, and all the remaining terms go to $0$ after dividing $e^{\bea}$. Hence, when $\gamma < \frac{1}{1-\rho}$, the leading order term in Equation~(\ref{e:att grad norm 1}) is the term $\bea^2 e^{2\bea \rho}  n^3 \rho^2(1-\rho)$; when $\gamma > \frac{1}{1-\rho}$, the leading order term is $e^{2 \bea}(d-1)n$. This proves Equation~(\ref{e:att grad norm phase 1}).
\end{proof}


\begin{lemma}\label{lem:att grad norm 2}
    For terms $\bfr_1,\bfr_2,\bfu_1,\bfu_2,\bfu_3,\bfu_4$ in Lemma~\ref{lem:att norm chain 2}, we have that
        \begin{align}\label{e:att grad norm 2}
            \begin{split}
                &\sum_{i,j} (\bfu_1 + \bfu_2)^T (\bfr_1+\bfr_2) (\bfu_3 + \bfu_4)
                \\  &= \rho \bea^2 e^{2 \bea \rho} (e^{\bea}-e^{\bea \rho})n(n-1)(n\rho + (1-\rho))(1-\rho)
                \\  & \quad +\bea^2 e^{3\bea \rho} n(n-1)(n\rho + (1-\rho))^2(1-\rho)
                \\  & \quad + \bea e^{\bea(2\rho+1)} n(n-1)(n\rho + (1-\rho))(1-\rho)
                \\  & \quad + \bea  e^{2\bea \rho} (e^{\bea}-e^{\bea \rho}) (\bea\rho+1) n(n-1)(1-\rho^2)
                \\  & \quad + \bea  e^{3\bea \rho} n(n-1)(n\rho + (1-\rho))(\bea (1-\rho^2) + (1-\rho)).
            \end{split}
        \end{align}
As a corollary, when we pick $\bea = \gamma \log n$, we have the following phase transition limits as $n \to +\infty$:
    \begin{align}\label{e:att grad norm phase 2}
         \frac{1}{nZ^3} \sum_{i,j} (\bfu_1 + \bfu_2)^T (\bfr_1+\bfr_2) (\bfu_3 + \bfu_4) = 
            \begin{cases}
                \bea^2 \rho^2 (1-\rho) +o_n(1)  & \text{if $\gamma < \frac{1}{1-\rho}$},
                \\ \frac{\bea^2 \rho^2 (1-\rho)}{4} +o_n(1) & \text{if $\gamma = \frac{1}{1-\rho}$}.
            \\ 0 +o_n(1) & \text{if $\gamma > \frac{1}{1-\rho}$}.
            \end{cases}
    \end{align}
\end{lemma}
\begin{proof}[Proof of Lemma~\ref{lem:att grad norm 2}]
    We expand each term in Lemma~\ref{lem:att grad norm 2} and also apply Lemma~\ref{lem:basic equalities for VW simplex} to each term. We first estimate terms involving $\bfu_1$.
        \begin{align}
            \begin{split}
                &\sum_{i,j}\bfu_1 ^T \bfr_1 \bfu_3 = \sum_i \bea^2 e^{2 \bea \rho} (e^{\bea}-e^{\bea \rho}) \left(\proj_{y_i} \sumv \right)^T \left( \sumt - y_i \otimes (\sumt y_i)  \right)y_i
                \\  &=  \bea^2 e^{2 \bea \rho} (e^{\bea}-e^{\bea \rho}) \sum_i \left(\proj_{y_i} \sumv \right)^T \sumt y_i = \rho \bea^2 e^{2 \bea \rho} (e^{\bea}-e^{\bea \rho}) \sum_i \left(\proj_{y_i} \sumv \right)^T \sumv
                \\  &= \rho \bea^2 e^{2 \bea \rho} (e^{\bea}-e^{\bea \rho})n(n-1)(n\rho + (1-\rho))(1-\rho),
            \end{split}
        \end{align}
    where the second and the third equality is because $\langle \proj_{y_i} \sumv , y_i \rangle = 0$.
    \begin{align}
        \begin{split}
            &\sum_{i,j} \bfu_1 ^T \bfr_1 \bfu_4 = \sum_i \bea^2 e^{3\bea \rho} \left(\proj_{y_i} \sumv \right)^T \left( \sumt - y_i \otimes (\sumt y_i)  \right) \sumv
            \\  &= \bea^2 e^{3\bea \rho}  \sum_i \left(\proj_{y_i} \sumv \right)^T \sumt \sumv = \bea^2 e^{3\bea \rho} (n\rho + (1-\rho)) \sum_i \|\proj_{y_i} \sumv \|^2
            \\  &= \bea^2 e^{3\bea \rho} n(n-1)(n\rho + (1-\rho))^2(1-\rho),
        \end{split}
    \end{align}
where the second equality is because $\langle \proj_{y_i} \sumv , y_i \rangle = 0$.
    \begin{align}
        \begin{split}
            &\sum_{i,j} \bfu_1 ^T \bfr_2 \bfu_3 = \sum_i \bea e^{\bea (\rho + 1)} (e^{\bea} -e^{\bea \rho}) \left(\proj_{y_i} \sumv \right)^T \left((-y_i+\bea \proj_{y_i} y_i) \otimes y_i+I_d  \right) y_i
            =0,
        \end{split}
    \end{align}
where the second equality is because $\langle \proj_{y_i} \sumv , y_i \rangle = 0$ and $\proj_{y_i} y_i = 0$.
    \begin{align}
        \begin{split}
            &\sum_{i,j} \bfu_1 ^T \bfr_2 \bfu_4 = \sum_i \bea e^{\bea(2\rho+1)} \left(\proj_{y_i} \sumv \right)^T \left((y_i+\bea \proj_{y_i} y_i) \otimes y_i+I_d  \right) \sumv
            \\  &= \bea e^{\bea(2\rho+1)} \sum_i \|\proj_{y_i} \sumv \|^2 = \bea e^{\bea(2\rho+1)} n(n-1)(n\rho + (1-\rho))(1-\rho),
        \end{split}
    \end{align}
where the second equality is because $\langle \proj_{y_i} \sumv , y_i \rangle = 0$ and $\proj_{y_i} y_i = 0$.

Next, we estimate the terms involving $\bfu_2$. We first recall that $\bfu_2 = \bea  e^{\bea \langle y_j,y_i \rangle}\left(\proj_{y_i} y_j \right)$. Because $\proj_{y_i} y_j = 0$ when $i=j$, we can just replace $e^{\bea \langle y_j,y_i \rangle}$ with $e^{\bea \rho}$ in $\bfu_2$, i,e, $\bfu_2 = \bea  e^{\bea \rho}\left(\proj_{y_i} y_j \right)$. Hence,
    \begin{align}
        \begin{split}
            \sum_{i,j} \bfu_2 ^T \bfr_1 \bfu_3 = 0, \quad \sum_{i,j} \bfu_2 ^T \bfr_1 \bfu_4 = 0,
        \end{split}
    \end{align}
because $\delta_{ij} (\proj_{y_i} y_j) =0$ for any $i,j$ in $\bfu_2 ^T \bfr_1$.
    \begin{align}
        \begin{split}
            &\sum_{i,j} \bfu_2 ^T \bfr_2 \bfu_3 = \sum_{i,j} \bea  e^{\bea (\rho+\langle y_j , y_i\rangle)} (e^{\bea}-e^{\bea \rho})\left(\proj_{y_i} y_j \right)^T \left((-y_i+\bea \proj_{y_i} y_j) \otimes y_i+I_d  \right) y_j
            \\  &= \bea  e^{2\bea \rho} (e^{\bea}-e^{\bea \rho}) \sum_{i \neq j}\left(\proj_{y_i} y_j \right)^T \left((-y_i+\bea \proj_{y_i} y_j) \rho+y_j  \right) 
            \\  &= \bea  e^{2\bea \rho} (e^{\bea}-e^{\bea \rho}) (\bea\rho+1) \sum_{i \neq j} \|\proj_{y_i} y_j\|^2 
            \\  &= \bea  e^{2\bea \rho} (e^{\bea}-e^{\bea \rho}) (\bea\rho+1) n(n-1)(1-\rho^2).
        \end{split}
    \end{align}
where the second equality is because $\proj_{y_i} y_j \neq 0$ only when $i \neq j$, on which $\langle y_j,y_i\rangle = \rho$, and the third equality is because $\langle \proj_{y_i} y_j , y_i \rangle = 0$.
    \begin{align}
        \begin{split}
            &\sum_{i,j} \bfu_2 ^T \bfr_2 \bfu_4 = \sum_{i,j} \bea  e^{\bea (2\rho+\langle y_j , y_i\rangle)} \left(\proj_{y_i} y_j \right)^T \left((-y_i+\bea \proj_{y_i} y_j) \otimes y_i+I_d  \right) \sumv
            \\  &= \bea  e^{3\bea \rho}\sum_{i \neq j}   \left(\bea \|\proj_{y_i} y_j\|^2 (n\rho + (1-\rho))+ \left(\proj_{y_i} y_j \right)^T \sumv  \right) 
            \\  &= \bea  e^{3\bea \rho}\sum_{i \neq j}   \left(\bea (1-\rho^2) (n\rho + (1-\rho))+ (1-\rho)(n\rho + (1-\rho))  \right) 
            \\ &=\bea  e^{3\bea \rho} n(n-1)(n\rho + (1-\rho))(\bea (1-\rho^2) + (1-\rho)).
        \end{split}
    \end{align}
where the second equality is because $\proj_{y_i} y_j \neq 0$ only when $i \neq j$, on which $\langle y_j,y_i\rangle = \rho$, and the third equality is because $\langle \proj_{y_i} y_j , y_i \rangle = 0$.

The proof for Equation~(\ref{e:att grad norm phase 2}) is similar to the proof for Equation~(\ref{e:att grad norm phase 1}) in Lemma~\ref{lem:att grad norm 1}. Notice that when $\gamma < \frac{1}{1-\rho}$, we need to pick up terms involving $ne^{\bea \rho}$, and the leading order term in Equation~(\ref{e:att grad norm 2}) is the one in the second line of Equation~(\ref{e:att grad norm 2}), which is $\bea^2 n^4 e^{3\bea \rho} \rho^2 (1-\rho)$; when $\gamma > \frac{1}{1-\rho}$, after diving $nZ^3$, all terms in Equation~(\ref{e:att grad norm 2}) are $o_n(1)$ terms.
\end{proof}


\begin{lemma}\label{lem:att grad norm 3}
    For terms $\bfu_1,\bfu_2,\bfu_3,\bfu_4$ in Lemma~\ref{lem:att norm chain 2}, we have that
        \begin{align}\label{e:att grad norm 3}
            \begin{split}
                &\sum_{i,j}\|\bfu_1 + \bfu_2\|^2\|\bfu_3 + \bfu_4\|^2 
                \\ &= \bea^2 e^{2\bea \rho}  n(n-1)(n\rho+2)(1-\rho) 
                \\  & \quad \cdot \left[ (e^{\bea}-e^{\bea \rho})^2+ 2 e^{\bea \rho}(e^{\bea}-e^{\bea \rho}) (n\rho + (1-\rho))+ e^{2\bea \rho} n (n\rho + (1-\rho)) \right].
            \end{split}
        \end{align}
As a corollary, when we pick $\bea = \gamma \log n$, we have the following phase transition limits as $n \to +\infty$:
    \begin{align}\label{e:att grad norm phase 3}
         \frac{1}{nZ^4} \sum_{i,j}\|\bfu_1 + \bfu_2\|^2\|\bfu_3 + \bfu_4\|^2  = 
            \begin{cases}
                \bea^2 \rho^2 (1-\rho) +o_n(1)  & \text{if $\gamma < \frac{1}{1-\rho}$},
            \\ \frac{\bea^2 \rho (1-\rho)(1+3\rho)}{16} +o_n(1) & \text{if $\gamma = \frac{1}{1-\rho}$},
            \\ 0 +o_n(1) & \text{if $\gamma > \frac{1}{1-\rho}$}.
            \end{cases}
    \end{align}
\end{lemma}
\begin{proof}[Proof of Lemma~\ref{lem:att grad norm 3}]
    We notice that $\langle \bfu_1 ,\bfu_2 \rangle = 0$ because $\delta_{ij} \proj_{y_i} y_j = 0$ for any $i,j$. So, 
    \begin{align}
        \begin{split}
            &\|\bfu_1+\bfu_2\|^2 =  \delta_{ij} \bea^2 e^{2\bea \rho}    \|\proj_{y_i} \sumv \|^2  + \bea^2  e^{2\bea \langle y_j,y_i \rangle}\|\proj_{y_i} y_j \|^2
        \\  &= \delta_{ij}\bea^2 e^{2\bea \rho}     (n-1)(n\rho + (1-\rho))(1-\rho) + (1-\delta_{ij})\bea^2 e^{2\bea \rho}(1-\rho^2),
        \end{split}
    \end{align}
where the second equality is because $e^{2\bea \langle y_j,y_i \rangle}\|\proj_{y_i} y_j \|^2 \neq 0$ only if $i \neq j$, on which $e^{2\bea \langle y_j,y_i \rangle}\|\proj_{y_i} y_j \|^2 = e^{2\bea \rho}(1-\rho^2)$.
    \begin{align}
        \begin{split}
            &\|\bfu_3+\bfu_4\|^2 = (e^{\bea}-e^{\bea \rho})^2 + 2 e^{\bea \rho}(e^{\bea}-e^{\bea \rho})  \langle \sumv , y_j \rangle + e^{2\bea \rho}\|\sumv\|^2 
        \\  & =  (e^{\bea}-e^{\bea \rho})^2+ 2 e^{\bea \rho}(e^{\bea}-e^{\bea \rho}) (n\rho + (1-\rho))+ e^{2\bea \rho} n (n\rho + (1-\rho)),
        \end{split}
    \end{align}
which is independent of $i,j$.
Hence,
    \begin{align}
        \begin{split}
            &\sum_{i,j}\|\bfu_1 + \bfu_2\|^2\|\bfu_3 + \bfu_4\|^2 
            \\  &= \left[\bea^2 e^{2\bea \rho}  n(n-1)(n\rho + (1-\rho))(1-\rho) + n(n-1)\bea^2 e^{2\bea \rho}(1-\rho^2)\right] \|\bfu_3 + \bfu_4\|^2
            \\  &= \bea^2 e^{2\bea \rho}  n(n-1)(n\rho+2)(1-\rho) \|\bfu_3 + \bfu_4\|^2
            \\ &= \bea^2 e^{2\bea \rho}  n(n-1)(n\rho+2)(1-\rho) 
            \\  & \quad \cdot \left[ (e^{\bea}-e^{\bea \rho})^2+ 2 e^{\bea \rho}(e^{\bea}-e^{\bea \rho}) (n\rho + (1-\rho))+ e^{2\bea \rho} n (n\rho + (1-\rho)) \right].
        \end{split}
    \end{align}

The proof for Equation~(\ref{e:att grad norm phase 3}) is similar to the proof for Equation~(\ref{e:att grad norm phase 1}) in Lemma~\ref{lem:att grad norm 1}. Notice that when $\gamma < \frac{1}{1-\rho}$, we need to pick up terms involving $ne^{\bea \rho}$, and the leading order term in Equation~(\ref{e:att grad norm 3}) is $\bea^2 n^5 e^{4\bea \rho} \rho^2 (1-\rho)$; when $\gamma > \frac{1}{1-\rho}$, after dividing $nZ^4$, all terms in Equation~(\ref{e:att grad norm 3}) are $o_n(1)$ terms.
\end{proof}


\begin{proof}[Proof of Theorem~\ref{thm:att propagation gradient norm 2}]
    As we have mentioned at the beginning of Appendix~\ref{sec:proof gradient norm 2}, we only need to focus the case when $\gamma \leq \frac{1}{1-\rho}$, which follows directly from Lemma~\ref{lem:att grad norm 1}, Lemma~\ref{lem:att grad norm 2}, and Lemma~\ref{lem:att grad norm 3}. We notice that, in these three lemmas, the leading order terms are the same, $\bea^2 \rho^2 (1-\rho)$, which cancels in Equation~(\ref{e:jacobian norm square expansion simplex}). Hence, when $\gamma < \frac{1}{1-\rho}$, $\frac{1}{nd}\|\nabla_X X'\|^2=  0+o_n(1)$. When $\gamma = \frac{1}{1-\rho}$, we also only need to use the corresponding cases in these three lemmas and combine them in Equation~(\ref{e:jacobian norm square expansion simplex}) to get the conclusion in Theorem~\ref{thm:att propagation gradient norm 2}. One remark is that under Assumption~\ref{a:simplex}, we have that $n \leq d$ implicitly. So, when $\gamma = \frac{1}{1-\rho}$, terms in Equation~(\ref{e:jacobian norm square expansion simplex}) involving $\frac{\bea^2}{d} = \frac{\gamma ^2 (\log n)^2}{d}$ also become $o_n(1)$.
\end{proof}


\section{Modified Assumptions with More Middle Phases}\label{sec:more transition phases}

In this section, we modify Assumption \ref{a:equal angle}, so that we can prove the existence of three different phases like Lemma \ref{lem:Z_i asymptotics}, Theorem \ref{thm:att limit phase}, Theorem \ref{thm:att propagation gradient norm}. We remark that we only showed the existence of two phases (two extrema) in Lemma \ref{lem:Z_i asymptotics}, Theorem \ref{thm:att limit phase}, Theorem \ref{thm:att propagation gradient norm}, but it doesn't mean under Assumption \ref{a:equal angle}, there is no other transition phase between these two phases (two extrema). Under the following Assumption \ref{a:three phase}, we can show there are indeed at least three phases. Recall that for any $i \in \llbracket 1 , n \rrbracket$, we defined $y_i = N(x_i)$.

\begin{assumption}\label{a:three phase}
\indent
    \begin{itemize}
        \item For any $i \in \llbracket 1 , n \rrbracket$, $\|x_i\|^2 \in [q_1,q_2]$ for some positive constants $q_1 \leq q_2$.
        \item There is a $\tau \in (0,1]$, four positive constants $\rho_3 ,\rho_4, \kappa_3,\kappa_4$ with $\rho_3 \leq \rho_4$, $\kappa_3 \leq \kappa_4$, and $\rho_4 <1$, such that for any $i \in \llbracket 1 , n \rrbracket$, if we define 
            \begin{align}
                \cK_i = \left\{ m \neq i \ | \ \langle y_m , y_i \rangle \in [\rho_3,\rho_4] \right\},
            \end{align}
        then we have that 
            \begin{align}
                \kappa_3 \leq \frac{\left| \cK_i \right|}{n^\tau} \leq \kappa_4.
            \end{align}
        \item For any $i \in \llbracket 1 , n \rrbracket$ and any $j \notin \cK_i \cup \{i\}$, $\langle y_i , y_j \rangle \in [\rho_1,\rho_2]$ for some nonnegative constants $\rho_1 ,\rho_2$ satisfying $\rho_1 \leq \rho_2<\rho_3 \leq \rho_4$.
        \item For technical reason, we further assume that $(1-\tau)(1-\rho_2) + \rho_2 < \rho_3$.
    \end{itemize}
\end{assumption}

\begin{lemma}\label{lem:Z_i three phase}
    Let $\bea = \gamma \log n$ where $\gamma$ is a positive constant. Under Assumption \ref{a:three phase} and Equation~(\ref{e:att update}), for any $i \in \llbracket 1 , n \rrbracket$,
    \begin{align}\label{e:Z_i 3 phases}
        Z_i = 
        \begin{cases}
           (1+o_n(1)) \cdot \left(\sum_{m \notin \cK_i \cup \{i\}} e^{a_{im}}\right)   & \text{if $\gamma < \min \left\{\frac{1}{1-\rho_1}, \frac{1-\tau}{\rho_4 - \rho_1} \right\}$},
           \\ (1+o_n(1)) \cdot \left(\sum_{m \in \cK_i } e^{a_{im}}\right)   & \text{if $\frac{1-\tau}{\rho_3-\rho_2} < \gamma <  \frac{\tau}{1-\rho_3} $},
            \\ (1+o_n(1)) \cdot e^{\bea}  & \text{if $\gamma > \max \left\{\frac{1}{1-\rho_2}, \frac{\tau}{1 - \rho_4} \right\}$},
        \end{cases}
    \end{align}
where the terms $o_n(1)$ go to $0$ as $n \to +\infty$ with speeds independent of $i$ but only depending on $\gamma,\rho_1, \rho_2,\rho_3, \rho_4,\tau,\kappa_3,\kappa_4$.
\end{lemma}
\begin{proof}
    The proof is similar to Lemma \ref{lem:Z_i asymptotics}. We notice that 
        \begin{align}
            \begin{split}
                Z_i &= e^{\bea} + \sum_{m \in \cK_i} e^{a_{im}} + \sum_{m \notin \cK_i \cup \{i\} } e^{a_{im}}
                \\  &= n^{\gamma} + \sum_{m \in \cK_i} n^{\gamma \langle y_i , y_m \rangle } + \sum_{m \notin \cK_i \cup \{i\} } n^{\gamma \langle y_i , y_m \rangle }.
            \end{split}
        \end{align}
    We also notice that $\kappa_3 n^{\tau} \leq |\cK_i| \leq \kappa_4 n^{\tau}$ according to Assumption \ref{a:three phase}. 
    Hence, 
        \begin{align}
            \kappa_3 n^{\tau + \gamma \rho_3}\leq |\cK_i| \cdot n^{\gamma \rho_3} \leq \sum_{m \in \cK_i} n^{\gamma \langle y_i , y_m \rangle } \leq |\cK_i| \cdot n^{\gamma \rho_4} \leq \kappa_4 n^{\tau + \gamma \rho_4},
        \end{align}
    and 
        \begin{align}
            (n-\kappa_4 n^{\tau}-1) \cdot n^{\gamma \rho_1} \leq \sum_{m \notin \cK_i \cup \{i\} } n^{\gamma \langle y_i , y_m \rangle } \leq (n-|\cK_i|-1) \cdot n^{\gamma \rho_2} \leq n^{1+\gamma \rho_2}.
        \end{align}
    When $\gamma < \min \left\{\frac{1}{1-\rho_1}, \frac{1-\tau}{\rho_4 - \rho_1} \right\}$, the leading order term in $Z_i$ is $\sum_{m \notin \cK_i \cup \{i\} } n^{\gamma \langle y_i , y_m \rangle }$; when $\frac{1-\tau}{\rho_3-\rho_2} < \gamma <  \frac{\tau}{1-\rho_3} $, the leading order term in $Z_i$ is $\sum_{m \in \cK_i} n^{\gamma \langle y_i , y_m \rangle }$; when $\gamma > \max \left\{\frac{1}{1-\rho_2}, \frac{\tau}{1 - \rho_4} \right\}$, the leading order term in $Z_i$ is $n^\gamma$. We also remark that the last assumption in Assumption \ref{a:three phase} is to ensure the existence of the middle phase, i.e., $\frac{1-\tau}{\rho_3-\rho_2} < \gamma <  \frac{\tau}{1-\rho_3} $. This finishes the proof for Lemma \ref{lem:Z_i three phase} by similar arguments as in Lemma \ref{lem:Z_i asymptotics}.
\end{proof}
A direct corollary of Lemma \ref{lem:Z_i three phase} is the following theorem.
\begin{theorem}\label{thm:att three phase}
    Under Assumption \ref{a:three phase} and Equation~(\ref{e:att update}) we have the following phase transition phenomena: let $\bea = \gamma \log n$ where $\gamma$ is a positive constant. For any $i \in \llbracket 1 , n \rrbracket$ the updating dynamics Equation~(\ref{e:att update}) can be written as
        \begin{align}
            x_i ' =  \ala x_i + \begin{cases}
           \frac{\sum_{m \notin \cK_i \cup \{i\}} e^{a_{im}} y_m}{\sum_{m \notin \cK_i \cup \{i\}} e^{a_{im}}} + \bfo_n(1)  & \text{if $\gamma < \min \left\{\frac{1}{1-\rho_1}, \frac{1-\tau}{\rho_4 - \rho_1} \right\}$},
           \\ \frac{\sum_{m \in \cK_i} e^{a_{im}} y_m}{\sum_{m \in \cK_i} e^{a_{im}}} + \bfo_n(1)   & \text{if $\frac{1-\tau}{\rho_3-\rho_2} < \gamma <  \frac{\tau}{1-\rho_3} $},
            \\ y_i + \bfo_n(1)  & \text{if $\gamma > \max \left\{\frac{1}{1-\rho_2}, \frac{\tau}{1 - \rho_4} \right\}$},
        \end{cases}
        \end{align}
The terms $\bfo_n(1)$ represent vectors in $\R^d$ with norms going to $0$ as $n \to +\infty$, with a speed independent of $i$ but only depending on $\gamma,\rho_1, \rho_2,\rho_3, \rho_4,\tau,\kappa_3,\kappa_4$.
\end{theorem}
The proof of Theorem \ref{thm:att three phase} is similar to Lemma \ref{lem:Z_i three phase} so we omit its proof.

\section{Analysis of the $\beta_n \asymp \sqrt{\log n}$ scaling for i.i.d.\ Gaussian Scores}
\label{app:gaussian-supercritical}

This appendix provides a short heuristic derivation of the $\beta_n \asymp \sqrt{\log n}$ scaling for softmax attention when the raw attention scores  $a_{1},\dots,a_{n}$ are modeled as independent $\mathcal{N}(0,1)$ random variables. 
The purpose is to contrast the behavior of this Gaussian setting with the geometric setting analyzed in the main text, where pairwise score gaps remain $O(1)$ and the critical scale becomes $\beta_n \asymp \log n$.

Let $a_1,\dots,a_n$ be i.i.d.\ $\mathcal{N}(0,1)$ and denote by
\[
a_1^\downarrow \ge a_2^\downarrow \ge \cdots \ge a_n^\downarrow
\]
their order statistics. Set
\[
t_n := \Phi^{-1}\!\left(1 - \frac{1}{n}\right) \sim \sqrt{2\log n},
\]
where $\Phi$ is the standard normal CDF. It is classical (see Theorem~2.1.1 in~\cite{de2006extreme})
that for any fixed $k \ge 2$,
\begin{equation}
\label{eq:ev-spacing}
\bigl(t_n(a_1^\downarrow - a_2^\downarrow), \dots, t_n(a_1^\downarrow - a_k^\downarrow)\bigr)
\;\Rightarrow\;
(E_1,\; E_1 + E_2,\; \dots,\; E_1+\cdots+E_{k-1}),
\end{equation}
where $(E_i)_{i\ge1}$ are i.i.d.\ $\mathrm{Exp}(1)$ random variables. Thus the gap scale between neighboring scores in the Gaussian model is $1/\sqrt{\log n}$.

Define the softmax weights and the top weight
\[
A_j(n) := \frac{\exp(\beta_n a_j)}{\sum_{k=1}^n \exp(\beta_n a_k)},
\qquad
A_1^\downarrow(n) := \max_{1\le j\le n} A_j(n).
\]
Using Equation~(\ref{eq:ev-spacing}), one can write heuristically
\begin{equation}
\label{eq:softmax-denominator}
A_1^\downarrow(n)
\approx 
\frac{1}{
1 + \sum_{k=2}^\infty 
\exp\!\bigl\{ - (\beta_n/\sqrt{2\log n})\, S_{k-1} \bigr\}
},
\qquad 
S_{k-1} := E_1 + \cdots + E_{k-1} \,.
\end{equation}
where we have replaced the finite sum by an infinite series, which captures the leading asymptotics. The critical scaling $\sqrt{\log n}$ shows up in Equation~(\ref{eq:softmax-denominator}). For example, in the supercritical regime $\beta_n \gg \sqrt{\log n}$, we have the following:
\begin{proposition}[Supercritical regime]
\label{prop:supercritical-gaussian}
If 
\[
\frac{\beta_n}{\sqrt{\log n}} \;\longrightarrow\; \infty,
\]
then $A_1^\downarrow(n) \rightarrow 1$ as $n\to\infty$. In other words, the attention weights concentrate on the top-scoring token.
\end{proposition}

\begin{proof}[Sketch of proof]
In Equation~(\ref{eq:softmax-denominator}), we know the sum
\[
Z = \sum_{k=2}^\infty 
\exp\!\bigl\{ - (\beta_n/\sqrt{2\log n})\, S_{k-1} \bigr\}
\]
is almost surely finite. Since $\beta_n/\sqrt{\log n} \to \infty$, every such term in $Z$ vanishes in probability. Hence the denominator in Equation~(\ref{eq:softmax-denominator}) tends to $1$, and $A_1^\downarrow(n)\to 1$ as $n\to\infty$.
\end{proof}

Proposition~\ref{prop:supercritical-gaussian} shows that the softmax enters a regime in 
which a single index captures asymptotically all attention mass as soon as 
$\beta_n \gg \sqrt{\log n}$.  
The critical scale is thus determined by the requirement
\[
\beta_n(a_1^\downarrow - a_k^\downarrow) = O(k)
\quad\Longleftrightarrow\quad
\beta_n \asymp \sqrt{\log n}.
\]


In the geometric models analyzed in the main text (simplex and almost-simplex assumptions), the gaps between the top pairwise inner products remain of order $O(1)$ as $n \to \infty$. Consequently, the balancing argument in \Cref{sec:intro} yields the critical scaling $\beta_n \asymp \log n$. This stands in clear contrast to REM-type Gaussian models, where the top gaps shrink at the $1/\sqrt{\log n}$ scale and thus produce the critical regime $\beta_n \asymp \sqrt{\log n}$. This explains why REM-type Gaussian models lead to the $\sqrt{\log n}$ scale, whereas geometric models naturally produce the $\log n$ scale that aligns with many practical strategies to avoid mixing in long-context attention.


\bibliographystyle{alpha}
\bibliography{references}

@article{rigollet2025mean,
  title={The mean-field dynamics of transformers},
  author={Rigollet, Philippe},
  journal={arXiv preprint arXiv:2512.01868},
  year={2025}
}

@book{de2006extreme,
  title={Extreme value theory: an introduction},
  author={De Haan, Laurens and Ferreira, Ana},
  year={2006},
  publisher={Springer}
}

@book{lions1971optimal,
  title={Optimal control of systems governed by partial differential equations},
  author={Lions, Jacques Louis},
  volume={170},
  year={1971},
  publisher={Springer}
}

@article{rumelhart1986learning,
  title={Learning representations by back-propagating errors},
  author={Rumelhart, David E and Hinton, Geoffrey E and Williams, Ronald J},
  journal={nature},
  volume={323},
  number={6088},
  pages={533--536},
  year={1986},
  publisher={Nature Publishing Group UK London}
}

@article{bruno2025multiscaleanalysismeanfieldtransformers,
      title={A multiscale analysis of mean-field transformers in the moderate interaction regime}, 
      author={Giuseppe Bruno and Federico Pasqualotto and Andrea Agazzi},
      year={2025},
      journal={{NeurIPS}},
}

@misc{JaiMizSaw25,
	author = {Vishesh Jain and Clayton Mizgerd and Mehtaab Sawhney},
	note = {arXiv:2501.12205},
	title = {The random graph process is globally synchronizing},
	year = {2025}
}

@inproceedings{liu2021swin,
  title={Swin transformer: Hierarchical vision transformer using shifted windows},
  author={Liu, Ze and Lin, Yutong and Cao, Yue and Hu, Han and Wei, Yixuan and Zhang, Zheng and Lin, Stephen and Guo, Baining},
  booktitle={Proceedings of the IEEE/CVF international conference on computer vision},
  pages={10012--10022},
  year={2021}
}

@article{beltagy2020longformer,
  title={Longformer: The long-document transformer},
  author={Beltagy, Iz and Peters, Matthew E and Cohan, Arman},
  journal={arXiv preprint arXiv:2004.05150},
  year={2020}
}

@article{Derrida1981,
  title = {{Random-energy model: An exactly solvable model of disordered systems}},
  author = {Derrida, B.},
  journal = {Phys. Rev. B},
  volume = {24},
  issue = {5},
  pages = {2613--2626},
  year = {1981},
  publisher = {American Physical Society},
  doi = {10.1103/PhysRevB.24.2613}
}

@misc{karagodin2025normalization,
      title={Normalization in Attention Dynamics},
      author={Nikita Karagodin and Shu Ge and Yury Polyanskiy and Philippe Rigollet},
      year={2025},
      note={arXiv:2510.22026}
}

@misc{chen2025residual,
      title={Residual connections provably mitigate oversmoothing in graph neural networks},
      author={Ziang Chen and Zhengjiang Lin and Shi Chen and Yury Polyanskiy and Philippe Rigollet},
      year={2025},
      note={arXiv:2501.00762}
}

@misc{chen2025quantitative,
      title={Quantitative Clustering in Mean-Field Transformer Models},
      author={Shi Chen and Zhengjiang Lin and Yury Polyanskiy and Philippe Rigollet},
      year={2025},
      note={arXiv:2504.14697}
}

@misc{giorlandino2025failuremodesdeeptransformers,
      title={Two failure modes of deep transformers and how to avoid them: a unified theory of signal propagation at initialisation}, 
      author={Alessio Giorlandino and Sebastian Goldt},
      year={2025},
      note={arXiv:2505.24333}
}

@misc{puvvada2025swan,
      title={Swan-gpt: An efficient and scalable approach for long-context language modeling},
      author={Krishna C Puvvada and Faisal Ladhak and Santiago Akle Serrano and Cheng-Ping Hsieh and Shantanu Acharya and Somshubra Majumdar and Fei Jia and Samuel Kriman and Simeng Sun and Dima Rekesh and others},
      year={2025},
      note={arXiv:2504.08719}
}

@misc{bai2023qwen,
      title={Qwen technical report},
      author={Jinze Bai and Shuai Bai and Yunfei Chu and Zeyu Cui and Kai Dang and Xiaodong Deng and Yang Fan and Wenbin Ge and Yu Han and Fei Huang and others},
      year={2023},
      note={arXiv:2309.16609}
}

@misc{peng2023yarn,
      title={Yarn: Efficient context window extension of large language models},
      author={Bowen Peng and Jeffrey Quesnelle and Honglu Fan and Enrico Shippole},
      year={2023},
      note={arXiv:2309.00071}
}

@misc{nakanishi2025scalable,
      title={Scalable-Softmax Is Superior for Attention},
      author={Ken M Nakanishi},
      year={2025},
      note={arXiv:2501.19399}
}

@misc{andrew25,
      title={Synchronization of mean-field models on the circle}, 
      author={Yury Polyanskiy and Philippe Rigollet and Andrew Yao},
      year={2025},
      note={arXiv:2507.22857}
}

@article{geshkovski2023mathematical,
  title={A mathematical perspective on transformers},
  author={Geshkovski, Borjan and Letrouit, Cyril and Polyanskiy, Yury and Rigollet, Philippe},
  journal={Bull. Amer. Math. Soc.},
  year={2025}
}

@article{geshkovski2023emergence,
  title={The emergence of clusters in self-attention dynamics},
  author={Geshkovski, Borjan and Letrouit, Cyril and Polyanskiy, Yury and Rigollet, Philippe},
  journal={Advances in Neural Information Processing Systems},
  volume={36},
  year={2024}
}

@inproceedings{bruno2024emergence,
title={Emergence of meta-stable clustering in mean-field transformer models},
author={Bruno, Giuseppe and Pasqualotto, Federico and Agazzi, Andrea},
booktitle={International Conference on Learning Representations},
year={2025}
}

@misc{geshkovski2024dynamic,
      title={Dynamic metastability in the self-attention model},
      author={Borjan Geshkovski and Hugo Koubbi and Yury Polyanskiy and Philippe Rigollet},
      year={2024},
      note={arXiv:2410.06833}
}

@misc{abdalla2022expander,
      title={Expander graphs are globally synchronising},
      author={Pedro Abdalla and Afonso S Bandeira and Martin Kassabov and Victor Souza and Steven H Strogatz and Alex Townsend},
      year={2022},
      note={arXiv:2210.12788}
}

@inproceedings{dong2021attention,
  title={Attention is not all you need: Pure attention loses rank doubly exponentially with depth},
  author={Dong, Yihe and Cordonnier, Jean-Baptiste and Loukas, Andreas},
  booktitle={International Conference on Machine Learning},
  pages={2793--2803},
  year={2021},
  organization={PMLR}
}

@inproceedings{he2016deep,
  title={Deep residual learning for image recognition},
  author={He, Kaiming and Zhang, Xiangyu and Ren, Shaoqing and Sun, Jian},
  booktitle={Proceedings of the IEEE Conference on Computer Vision and Pattern Recognition},
  pages={770--778},
  year={2016}
}

@article{noci2022signal,
  title={{Signal propagation in transformers: Theoretical perspectives and the role of rank collapse}},
  author={Noci, Lorenzo and Anagnostidis, Sotiris and Biggio, Luca and Orvieto, Antonio and Singh, Sidak Pal and Lucchi, Aurelien},
  journal={Advances in Neural Information Processing Systems},
  volume={35},
  pages={27198--27211},
  year={2022}
}

@misc{cowsik2024geometric,
      title={Geometric Dynamics of Signal Propagation Predict Trainability of Transformers},
      author={Aditya Cowsik and Tamra Nebabu and Xiao-Liang Qi and Surya Ganguli},
      year={2024},
      note={arXiv:2403.02579}
}

@misc{karagodin2024clustering,
      title={Clustering in causal attention masking},
      author={Nikita Karagodin and Yury Polyanskiy and Philippe Rigollet},
      year={2024},
      note={arXiv:2411.04990}
}

@article{hutchinson1989stochastic,
  title={A stochastic estimator of the trace of the influence matrix for Laplacian smoothing splines},
  author={Hutchinson, Michael F},
  journal={Communications in Statistics-Simulation and Computation},
  volume={18},
  number={3},
  pages={1059--1076},
  year={1989},
  publisher={Taylor \& Francis}
}

\end{document}